\documentclass[onefignum,onetabnum]{siamonline190516}

\usepackage{lipsum}
\usepackage{amsfonts}
\usepackage{graphicx}
\usepackage{wrapfig}
\usepackage{booktabs} 
\usepackage{numprint}
\usepackage{epstopdf}
\usepackage{algorithmic}
\ifpdf
  \DeclareGraphicsExtensions{.eps,.pdf,.png,.jpg}
\else
  \DeclareGraphicsExtensions{.eps}
\fi

\usepackage{enumitem}
\setlist[enumerate]{leftmargin=.5in}
\setlist[itemize]{leftmargin=.5in}

\newcommand{\BEAS}{\begin{eqnarray*}}
\newcommand{\EEAS}{\end{eqnarray*}}
\newcommand{\BEA}{\begin{eqnarray}}
\newcommand{\EEA}{\end{eqnarray}}
\newcommand{\BEQ}{\begin{equation}}
\newcommand{\EEQ}{\end{equation}}
\newcommand{\BIT}{\begin{itemize}}
\newcommand{\EIT}{\end{itemize}}
\newcommand{\BNUM}{\begin{enumerate}}
\newcommand{\ENUM}{\end{enumerate}}

\newcommand{\BA}{\begin{array}}
\newcommand{\EA}{\end{array}}

\newcommand{\ie}{{\it i.e.}}

\newcommand{\ones}{\mathbf 1}

\newcommand{\reals}{{\mathbb R}}

\newcommand{\Card}{\mathop{\bf Card}}
\newcommand{\Tr}{\mathop{\bf Tr}}
\newcommand{\diag}{\mathop{\bf diag}}

\newcommand{\onehalf}{\frac{1}{2}}
\newcommand{\Co}{{\mathop {\bf Co}}}

\newcommand{\dsp}{\displaystyle}

\newsiamremark{remark}{Remark}
\newsiamremark{assumption}{Assumption}
\newsiamremark{hypothesis}{Hypothesis}
\crefname{hypothesis}{Hypothesis}{Hypotheses}
\newsiamthm{claim}{Claim}
\def\WP{\mbox{\rm WP}(\phi)}
\def\lpf{\lambda_{\rm pf}}

\headers{Implicit Deep Learning}{El Ghaoui et al.}

\title{Implicit Deep Learning\thanks{Submitted to the editors, August 6, 2020.
\funding{This work was funded in part by \url{sumup.ai}, the National Science Foundation, the Pacific Earthquake Engineering Research Center, and Total S.A.}}}

\author{Laurent El Ghaoui\thanks{EECS and IEOR departments, UC Berkeley (\email{elghaoui@berkeley.edu}).} \and Fangda Gu\footnotemark[2] \and Bertrand Travacca\thanks{CEE department, UC Berkeley.} \and Armin Askari\footnotemark[2] \and Alicia Y. Tsai\footnotemark[2]}

\usepackage{amsopn}
\usepackage{tikz}
\usetikzlibrary{arrows,calc,positioning,shadows,shapes,matrix,fit}
\usepackage{verbatim}

\ifpdf
\hypersetup{
pdftitle={Implicit Deep Learning},
pdfauthor={TBD}
}
\fi

\begin{document}

\maketitle

\begin{abstract}
Implicit deep learning prediction rules generalize the recursive rules of feedforward neural networks. Such rules are based on the solution of a fixed-point equation involving a single vector of hidden features, which is thus only implicitly defined. The implicit framework greatly simplifies the notation of deep learning, and opens up many new possibilities, in terms of novel architectures and algorithms, robustness analysis and design, interpretability, sparsity, and network  architecture optimization.
\end{abstract}

\begin{keywords}
Deep learning, implicit models, Perron-Frobenius theory, robustness, adversarial attacks.
\end{keywords}

\begin{AMS}
690C26, 49M99, 65K10, 62M45, 26B10
\end{AMS}


\section{Introduction}
\label{sec:introduction}
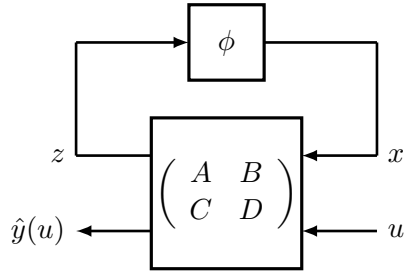
\begin{figure}[h]
\centering

\tikzstyle{system} = [rectangle, minimum height=2cm, minimum width=2cm]
\tikzstyle{activation} = [rectangle, minimum height=1cm, minimum width=1cm]
\tikzstyle{arrow} = [thick,->,>=stealth]
\tikzstyle{line} = [draw, very thick, -latex']
\tikzset{every picture/.style={line width=1pt}} 

\begin{tikzpicture}[auto,node distance=2cm,>=latex,every matrix/.append style={ampersand replacement=\&,matrix of nodes}]

\coordinate (topleft) at  (-2,0);	
\coordinate (phileft) at  (-0.5,0);	
\coordinate (bottomleft) at (-2,-1.5);	
\coordinate (systleft) at (-1,-1.5);	
\coordinate (topright) at  (2,0);	
\coordinate (phiright) at  (0.5,0);	
\coordinate (bottomright) at (2,-1.5);	
\coordinate (systright) at (1,-1.5);	
\coordinate (systy) at (-1,-2.5);
\coordinate (systu) at (1,-2.5);
\coordinate (A) at (-1,-1);
\coordinate (D) at (1,-3);
\coordinate (Af) at (-0.5,0.5);
\coordinate (Df) at (0.5,-0.5);

\node [activation, name=activation] {$\phi$};
\node [system, below of=activation, name=system] (system) 
{$\left(\begin{array}{cc} A & B \\ C & D \end{array}\right)$}; 


\draw[very thick] (A) rectangle (D);
\draw[very thick] (Af) rectangle (Df);

\draw[->]  (systy) -- ++ (-1,0) node[left]{$\hat{y}(u)$};
\draw[<-]  (systu) -- ++ (1,0) node[right]{$u$};

\draw (topleft) -- (bottomleft);
\draw (systleft) -- (bottomleft) node[left]{$z$};;
\draw[->] (topleft) -- (phileft);

\draw (topright) -- (bottomright);
\draw[<-] (systright) -- (bottomright) node[right]{$x$};
\draw (topright) -- (phiright);


\end{tikzpicture}

\caption{\label{fig:Figures_BlockDiag} 
A block-diagram view of an implicit model.
}
\end{figure}

\subsection{Implicit prediction rules} 
\label{sub:implicit_prediction_rules}
In this paper, we consider a new class of deep learning models that are based on implicit prediction rules. Such rules are not obtained via a recursive procedure through several layers, as in current neural networks. Instead, they are based on solving a fixed-point equation in some single ``state'' vector $x \in \reals^n$. Precisely, for a given input vector $u$, the predicted vector is
\begin{subequations}\label{eq:pred-rule-top}
\begin{align}
  \hat{y}(u) &=Cx+Du  \mbox{ [prediction equation]} \label{eq:pred-rule} \\
  x &= \phi(Ax+Bu) \mbox{ [equilibrium equation]} \label{eq:eq-eq} 
\end{align}
\end{subequations}
where $\phi \::\: \reals^n \rightarrow \reals^n$ is a nonlinear vector map (the ``activation'' map), and matrices $A,B,C,D$ contain model parameters. \cref{fig:Figures_BlockDiag} provides a block-diagram view of an implicit model, to be read from right to left, so as to be consistent with matrix-vector multiplication rules.

We can think of the vector $x \in \reals^n$ as a ``state'' corresponding to $n$ ``hidden'' features that are extracted from the inputs, based on the so-called equilibrium equation~\cref{eq:eq-eq}. In general, that equation cannot be solved in closed-form, and the model above provides $x$ only \emph{implicitly}. This equation is not necessarily well-posed, in the sense that it may not admit a solution, let alone a unique one; we discuss this important issue of well-posedness in \cref{sec:about_property_p}. 

For notational simplicity only, our rule does not contain any bias terms; we can easily account for those by considering the vector $(u,1)$  instead of $u$, thereby increasing the column dimension of $B$ by one. 

Perhaps surprisingly, as seen in \cref{sec:nn}, the implicit framework includes most current neural network architectures as special cases. Implicit models are a much wider class: they present much more capacity, as measured by the number of parameters for a given dimension of the hidden features; also, they allow for cycles in the network, which is not permitted under the current paradigm of deep networks.

Implicit rules open up the possibility of using novel architectures and prediction rules for deep learning, which are not based on any notion of ``network'' or ``layers'', as is classically understood. In addition, they allow one to consider rigorous approaches to challenging problems in deep learning, ranging from robustness analysis, sparsity and interpretability, and feature selection.


\subsection{Contributions and paper outline} 
\label{sub:contributions_and_paper_outline}
Our contributions in this paper, and its outline, are as follows.
\begin{itemize}
\item \emph{Well-posedness and composition} (\cref{sec:about_property_p}): In contrast with standard deep networks, implicit models may not be well-posed, in the sense that the equilibrium equation may have no or multiple solutions. We establish rigorous and numerically tractable conditions for implicit rules to be well-posed. These conditions are then used in the training problem, guaranteeing the well-posedness of the learned prediction rule. We also discuss the composition of implicit models, via cascade connections for example.

\item \emph{Implicit models of neural networks} (\cref{sec:nn}): We provide details on how to represent a wide variety of neural networks as implicit models, building on the composition rules of~\cref{sec:about_property_p}.

\item \emph{Robustness analysis} (\cref{sec:robustness}): We describe how to analyze the robustness properties of a given implicit model, deriving bounds on the state under input perturbations, and generating adversarial attacks. We also discuss which penalties to include into the training problem so as to encourage robustness of the learned rule.

\item \emph{Interpretability, sparsity, compression and deep feature selection} (\cref{sec:sparsity_and_topology_optimization}): Here we focus on finding appropriate penalties to use in order to improve properties such as model sparsity, or obtain feature selection. We also discuss the impact of model errors.

\item \emph{Training problem: formulations and algorithms} (\cref{sec:training_problem}): Informed by our previous findings, we finally discuss the corresponding training problem. Following the work of \cite{gu2018fenchel} and \cite{li2019lifted}, we represent activation functions using so-called Fenchel divergences, in order to relax the training problem into a more tractable form. We discuss several algorithms, including stochastic projected gradients, Frank-Wolfe, and block-coordinate descent.
\end{itemize}
Finally, \cref{sec:num} provides a few experiments supporting the theory put forth in this paper. Our final~\cref{sec:related_work} is devoted to prior work and references.




\subsection{Notation} 
For a matrix $U$, $|U|$ (resp. $U_+$) denotes the matrix with the absolute values (resp.\ positive part) of the entries of $U$. For a vector $v$, we denote by $\diag(v)$ the diagonal matrix formed with the entries of $v$; for a square matrix $V$, $\diag(V)$ is the vector formed with the diagonal elements of $V$. The notation $\ones$ refers to the vector of ones, with size inferred from context. The Hadamard (componentwise) product between two $n$-vectors $x,y$ is denoted $x \odot y$. We use $s_k(z)$ to denote the sum of the largest $k$ entries of a vector $z$. For a matrix $A$, and integers $p , q \ge 1$, we define the induced norm
\[
\|A\|_{p \rightarrow q} = \max_{\xi} \: \|A\xi\|_{q} ~:~ \|\xi\|_{p} \le 1.
\]
The case when $p = q = \infty$ corresponds to the $l_\infty$-induced norm of $A$, also known as its \emph{max-row-sum norm}:
\[
\|A\|_\infty:= \max_{i} \: \sum_j |A_{ij}|.
\]

We denote the set $\{1, \cdots,L\}$ compactly as $[L]$. For a $n$-vector partitioned into $L$ blocks, $z = (z_1,\ldots,z_L)$, with $z_l \in \reals^{n_l}$, $l \in [L]$, with $n_1+\ldots+n_L = n$, we denote by $\eta(z)$ the $L$-vector of norms:
\begin{align}\label{eq:vect-norms}
  \eta(z) := (\|z_1\|_{p_1}, \ldots, \|z_L\|_{p_L})^\top .
\end{align}

Finally, any square, non-negative matrix $M$ admits a real eigenvalue that is larger than the modulus of any other eigenvalue; this non-negative eigenvalue is the so-called  \emph{Perron-Frobenius eigenvalue} \cite{meyer2000matrix}, and is denoted $\lpf(M)$.


\section{Well-Posedness and Composition} 
\label{sec:about_property_p}

\subsection{Assumptions on the activation map}
\label{sub:ass-phi}
We restrict our attention to activation maps $\phi$ that obey a ``Blockwise LIPschitz'' (BLIP) continuity condition. This condition is satisfied for most popular activation maps, and arises naturally when ``composing'' implicit models (see~\cref{sub:composition}). Precisely, we assume that:
\begin{enumerate}
\item {\em Blockwise:} the map $\phi$ acts in a block-wise fashion, that is, there exist a partition of $n$: $n=n_1+\ldots+n_L$ such that for every vector partitioned into the corresponding blocks: $z = (z_1,\ldots,z_L)$ with $z_l \in \reals^{n_l}$, $l \in [L]$, we have $\phi(z) = (\phi_1(z_1),\ldots,\phi_L(z_L))$ for appropriate maps $\phi_l \::\: \reals^{n_l} \rightarrow \reals^{n_l}$, $l \in [L]$.

\item {\em Lipschitz:} For every $l \in [L]$, the maps $\phi_l$ are Lipschitz-continuous with constant $\gamma_l >0$ with respect to the $l_{p_l}$-norm for some integer $p_l \ge 1$:
\[
\forall \: u,v \in \reals^{n_l} ~:~ \|\phi_l(u)-\phi_l(v)\|_{p_l} \le \gamma_l\|u-v\|_{p_l}.
\]
\end{enumerate}
In the remainder of the paper, we refer to such maps with the acronym BLIP, omitting the dependence on the underlying structure information (integers $n_l$, $p_l$, $\gamma_l$, $l \in [L]$).
We shall consider a special case, referred to a COmponentwise Non-Expansive (CONE) maps, when $n_l = 1$, $\gamma_l = 1$, $l \in [L]$. Such CONE maps satisfy
\begin{align}\label{eq:cone-cond}
\forall \: u,v \in \reals^n ~:~ |\phi(u) - \phi(v) | \le |u-v|,
\end{align}
with inequality and absolute value taken componentwise. Examples of CONE maps include the ReLU (defined as $\phi(\cdot) = \max(0,\cdot)$) and its ``leaky'' variants, tanh, sigmoid, each applied componentwise to a vector input. Our model also allows for maps that do not operate componentwise, such as the softmax function, which operates on a $n$-vector $z$ as:
\begin{align}\label{eq:soft-max}
z \rightarrow \mbox{SoftMax}(z) := \left(\frac{e^{z_i}}{\sum_i e^{z_j}}\right)_{i \in [n]}, 
\end{align}
The softmax map is 1-Lipschitz-continuous with respect to the $l_1$-norm \cite{gao2017properties}.

\subsection{Well-posed matrices} 
\label{sub:definition}
We consider the prediction rule~\cref{eq:pred-rule} with input point $u \in \reals^p$ and predicted output vector $\hat{y}(u) \in \reals^q$.  The equilibrium equation~\cref{eq:eq-eq} does not necessarily have a well-defined, unique solution $x$, as~\cref{fig:ImplicitRule1D2} illustrates in a scalar case. In order to prevent this, we assume that the $n \times n$ matrix $A$ satisfies the following well-posedness property.

\begin{definition}[Well-posedness property]
The $n \times n$ matrix $A$ is said to be well-posed for $\phi$ (in short, $A \in \WP$) if, for any $n$-vector $b$, the equation in $x \in \reals^n$:
\begin{equation}\label{eq:implicit}
x = \phi(Ax+b) 
\end{equation} 
has a unique solution.
\end{definition}
\begin{figure}[h]
\centering
\includegraphics[width=.8\textwidth]{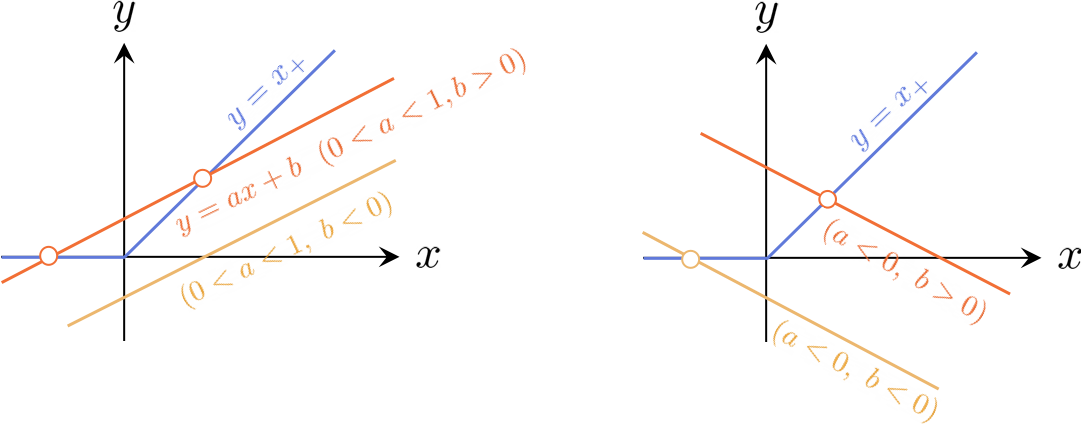}
\caption{\emph{Left:} equation $x = (ax+b)_+$ has two or no solutions, depending on the sign of $b$. \emph{Right:} for $a<0$, solution is unique for every $b$.}  \label{fig:ImplicitRule1D2}
\end{figure}
There are many classes of matrices that satisfy the well-posedness property. As seen next, strictly upper-triangular matrices are well-posed with respect to any activation map that acts componentwise; such a class arises when modeling feedforward neural networks as implicit models, as seen in \cref{sub:feedforward_neural_networks_are_a_special_case}. 

\subsection{Tractable sufficient conditions for well-posedness} 
\label{sub:tractable_sufficient_conditions}
Our goal now is to understand how we can constrain $A$ to have the well-posedness property, in a numerically tractable way. 

A sufficient condition is based on the contraction mapping theorem.  The following result focuses on the case when $\phi$ is componentwise non-expansive (CONE), as defined in~\cref{sub:ass-phi}. 

\begin{theorem}[PF sufficient condition for well-posedness with CONE activation]
\label{thm:pf-wp}
Assume that $\phi$ is a componentwise non-expansive (CONE) map, as defined in~\cref{sub:ass-phi}.  Then, $A$ is well-posed with respect to $\phi$ if $\lpf(|A|) <1$, in which case, for any $n$-vector $b$, the solution to the equation \cref{eq:implicit} can be computed via the fixed-point iteration 
\begin{equation}\label{eq:fixed-point-b}
x(0) =0, \;\; x(t+1) = \phi(Ax(t)+b) , \;\; t =0,1,2,\ldots.
\end{equation} 
\end{theorem}
The theorem is a direct consequence of the global contraction mapping theorem~\cite[p.83]{sastry2013nonlinear}. We provide a proof in \cref{app:proof_of_wellposedness}. A few remarks are in order.

\begin{remark}
The fixed-point iteration~\cref{eq:fixed-point-b} has linear convergence; each iteration is a matrix-vector product, hence the complexity is comparable to that of a forward pass through a network of similar size.
\end{remark}

\begin{remark}
The PF condition $\lpf(|A|) <1$ is not convex in $A$, but the convex condition $\|A\|_\infty < 1$, is sufficient, in light of the bound $\|A\|_{\infty} \ge \lpf(|A|)$.
\end{remark}

\begin{remark}
The PF condition of Theorem~\ref{thm:pf-wp} is conservative. For example, a triangular matrix $A$ is well-posed with respect to the ReLU and if only if $\diag(A) < \ones$, a consequence of the upcoming~\cref{thm:wp-tri}. The corresponding equilibrium equation can then be solved via the backward recursion
\[
x_n = \frac{(b_n)_+}{1-A_{nn}}, \;\; x_i = \frac{1}{1-A_{ii}}( b_i + \sum_{j>i} A_{ij}x_j)_+, \;\; i=n-1,\ldots,1.
\]
Such a matrix does not necessarily satisfy the PF condition; we can have in particular $A_{11}<-1$, which implies $\lpf(|A|)>1$.
\end{remark}

\begin{remark}
The well-posedness property is invariant under row and column permutation, provided $\phi$ acts componentwise. Precisely, 
if $A$ is well-posed with respect to a componentwise CONE map $\phi$, then for any $n \times n$ permutation matrix $P$, $PAP^\top$ is well-posed with respect to $\phi$. The PF sufficient condition is also invariant under row and column permutations.
\end{remark}

In some contexts, the map $\phi$ does not satisfy the componentwise non-expansiveness condition, but the weaker blockwise Lipschitz continuity (BLIP) defined in~\cref{sub:ass-phi}. The previous theorem can be extended to this case, as follows. We partition the $A$ matrix according to the tuple $(n_1,\ldots,n_L)$, into blocks $A_{ij} \in \reals^{n_i \times n_j}$, $1 \le i,j \le L$, and define a $L \times L$ matrix of induced norms, with elements for $l,h \in [L]$ given by
\begin{align}\label{eq:induced-norm-matrix}
(N(A))_{ij} := \|A\|_{p_j \rightarrow p_i} = \max_{\xi} \: \|A\xi\|_{p_i} ~:~ \|\xi\|_{p_j} \le 1.
\end{align}
\begin{theorem}[PF sufficient condition for well-posedness for BLIP activation]\label{thm:pf-wp-block}
Assume that $\phi$ satisfies the BLIP condition, as defined in~\cref{sub:ass-phi}.  Then, $A$ is well-posed with respect to $\phi$ if 
\[
\lpf(\Gamma N(A)) <1, 
\]
where $\Gamma := \diag(\gamma)$, with $\gamma$ the vector of Lipschitz constants, and $N(\cdot)$ is the matrix of induced norms defined in~\cref{eq:induced-norm-matrix}.
In this case, for any $n$-vector $b$, the solution to the equation \cref{eq:implicit} can be computed via the fixed-point iteration~\cref{eq:fixed-point-b}.
\end{theorem}
\begin{proof}
See~\cref{app:proof_of_wellposedness-block}.
\end{proof}

\subsection{Composition of implicit models} 
\label{sub:composition}
Implicit models can be easily composed via matrix algebra. Sometimes, the connection preserves well-posedness, thanks to the following result.

\begin{theorem}[Well-posedness of block-triangular matrices, componentwise activation]\label{thm:wp-tri}
Assume that the activation map $\phi$ acts componentwise. The upper block-triangular matrix
\[
A := \begin{pmatrix} A_{11} & A_{12} \\ 0 & A_{22} \end{pmatrix}
\]
with $A_{ii} \in \reals^{n_i \times n_i}$, $i=1,2$, is well-posed with respect to $\phi$ if and only if its the diagonal blocks $A_{11}, A_{22}$ are. 
\end{theorem}
\begin{proof}
See~\cref{app:wp-tri}.
\end{proof}
This result establishes the fact stated previously, that when $\phi$ is the ReLU, an upper-triangular matrix $A \in \WP$ if and only if $\diag(A) < \ones$. 
A similar result holds with the lower block-triangular matrix 
\[
A := \begin{pmatrix} A_{11} & 0 \\ A_{21} & A_{22} \end{pmatrix},
\]
where $A_{12}\in \reals^{n_1 \times n_2}$ is arbitrary. 
It is possible to extend this result to activation maps $\phi$ that satisfy the block Lipschitz continuity (BLIP) condition, in which case we need to assume that the partition of $A$ into blocks is consistent with that of $\phi$. As seen later, this feature arises naturally when composing implicit models from well-posed blocks.
\begin{theorem}[Well-posedness of block-triangular matrices, blockwise activation]\label{thm:wp-tri-block}
Assume that the matrix $A$ can be written as
\[
A := \begin{pmatrix} A_{11} & A_{12} \\ 0 & A_{22} \end{pmatrix}
\]
with $A_{ii} \in \reals^{n_i \times n_i}$, $i=1,2$, and $\phi$ acts blockwise accordingly, in the sense that there exist two maps $\phi_1,\phi_2$ such that $\phi((z_1,z_2)) = (\phi_1(z_1),\phi_2(z_2))$ for every $z_i \in \reals^{n_i}$, $i=1,2$. Then $A$ is well-posed with respect to $\phi$ if and only if for $i=1,2$, $A_{ii}$ is well-posed with respect to $\phi_i$. 
\end{theorem}

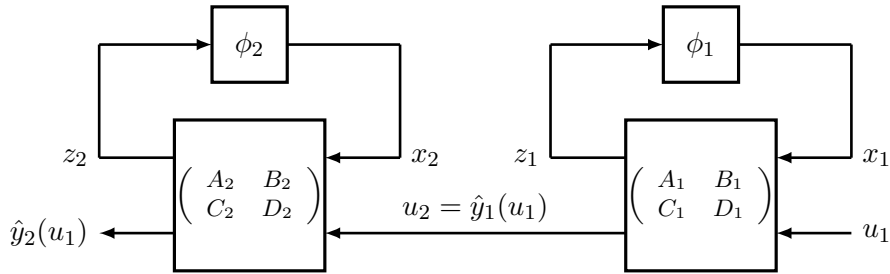
\begin{figure}[h]
\centering

\tikzstyle{system} = [rectangle,
minimum height=2cm, minimum width=2cm]
\tikzstyle{activation} = [draw, rectangle,
minimum height=1cm, minimum width=1cm]
\tikzstyle{arrow} = [thick,->,>=stealth]
\tikzstyle{line} = [draw, very thick, blue, -latex']
\tikzset{every picture/.style={line width=1pt}} 

\begin{tikzpicture}[auto,node distance=2cm,>=latex,every matrix/.append style={ampersand replacement=\&,matrix of nodes}]

\coordinate (topleft) at (-5,0); 
\coordinate (phileft) at (-3.5,0); 
\coordinate (bottomleft) at (-5,-1.5); 
\coordinate (systleft) at (-4,-1.5); 
\coordinate (topright) at (-1,0); 
\coordinate (phiright) at (-2.5,0); 
\coordinate (bottomright) at (-1,-1.5); 
\coordinate (systright) at (-2,-1.5); 
\coordinate (systy) at (-4,-2.5); 
\coordinate (systu) at (-2,-2.5); 
\coordinate (A) at (-4,-1); 
\coordinate (D) at (-2,-3); 
\coordinate (Af) at (-3.5,0.5); 
\coordinate (Df) at (-2.5,-0.5); 

\coordinate (topleft1) at (1,0); 
\coordinate (phileft1) at (2.5,0); 
\coordinate (bottomleft1) at (1,-1.5); 
\coordinate (systleft1) at (2,-1.5); 
\coordinate (topright1) at (5,0); 
\coordinate (phiright1) at (3.5,0); 
\coordinate (bottomright1) at (5,-1.5); 
\coordinate (systright1) at (4,-1.5); 
\coordinate (systy1) at (2,-2.5); 
\coordinate (systu1) at (4,-2.5); 
\coordinate (A1) at (2,-1); 
\coordinate (D1) at (4,-3); 
\coordinate (Af1) at (2.5,0.5); 
\coordinate (Df1) at (3.5,-0.5); 

\node[draw, very thick, fit={(A) (D)}, inner sep=0pt, label=center:\footnotesize $\left(\begin{array}{cc} A_2 & B_2 \\ C_2 & D_2 \end{array}\right)$] (a1) {};
\node[draw, very thick, fit={(Af) (Df)}, inner sep=0pt, label=center:$\phi_2$] (a2) {};
\draw[->]  (systy) -- ++ (-1,0) node[left]{$\hat{y}_2(u_1)$};
\draw (topleft) -- (bottomleft);
\draw (systleft) -- (bottomleft) node[left]{$z_2$};;
\draw[->] (topleft) -- (phileft);
\draw (topright) -- (bottomright);
\draw[<-] (systright) -- (bottomright) node[right]{$x_2$};
\draw (topright) -- (phiright);

\node[draw, very thick, fit={(A1) (D1)}, inner sep=0pt, label=center:\footnotesize$\left(\begin{array}{cc} A_1 & B_1 \\ C_1 & D_1 \end{array}\right)$] (a3) {};
\node[draw, very thick, fit={(Af1) (Df1)}, inner sep=0pt, label=center:$\phi_1$] (a4) {};
\draw[->]  (systy1) -- (systu) node[midway,above]{$u_2=\hat{y}_1(u_1)$};
\draw[<-]  (systu1) -- ++ (1,0) node[right]{$u_1$};
\draw (topleft1) -- (bottomleft1);
\draw (systleft1) -- (bottomleft1) node[left]{$z_1$};;
\draw[->] (topleft1) -- (phileft1);
\draw (topright1) -- (bottomright1);
\draw[<-] (systright1) -- (bottomright1) node[right]{$x_1$};
\draw (topright1) -- (phiright1);


\end{tikzpicture}

\caption{\label{fig:Figures_BlockDiag_Cascade} 
Cascade connection of two implicit models.
}
\end{figure}


Using the above results, we can preserve well-posedness of implicit models via composition. For example, given two models with matrix parameters $(A_i,B_i,C_i,D_i)$ and activation functions $\phi_i$, $i=1,2$, we can consider a ``cascaded'' prediction rule:
\[
\hat{y}_2 = C_2 x_2 +D_2 u_2 \mbox{ where } u_2 = \hat{y}_1 = C_1x_1 +D_1 u_1, \mbox{ where } x_i = \phi_i(A_ix_i+B_iu_i), \;\; i=1,2.
\]
The above rule can be represented as~\cref{eq:pred-rule}, with $x=(x_2,x_1)$, $\phi((z_2,z_1))=(\phi_2(z_2),\phi_1(z_1))$ and 
\[
\left(\begin{array}{c|c} 
A & B \\\hline C & D
\end{array}\right) = \left(\begin{array}{cc|c} 
A_2 & B_2C_1 & B_2D_1 \\
0 & A_1 & B_1 \\\hline 
C_2 & D_2C_1 & D_2D_1
\end{array}\right).
\]
Due to \cref{thm:wp-tri-block}, the cascaded rule is well-posed for the componentwise map with values $\phi(z_1,z_2)=(\phi_1(z_1),\phi_2(z_2))$ if and only if each rule is. 

A similar result holds if we put two or more well-posed models in parallel, and do a (weighted) sum the outputs. With the above notation, setting $\hat{y}(u_1,u_2) = \hat{y}_1(u_1) + \hat{y}_2(u_2)$ leads to a new implicit model that is also well-posed. Other possible connections include concatenation: $\hat{y}(u) = (\hat{y}_1(u),  \hat{y}_2(u))$, and affine transformations (a special case of cascade connection where one of the systems has no activation). We leave the details to the reader.

In both cascade and parallel connections, the triangular structure of the matrix $A$ of the composed system ensures that the PF sufficient condition for well-posedness is satisfied for the composed system if and only if it holds for each sub-system.

Multiplicative connections are in general are not Lipschitz-continuous, unless the inputs are bounded. Precisely, consider two activation maps $\phi_i$ that are Lipschitz-continuous with constant $\gamma_i$ and are bounded, with $|\phi_i(v)| \le c_i$ for every $v$, $i=1,2$; then, the multiplicative map
\[
(u_1,u_2) \in \reals^2 \rightarrow \phi(u) = \phi_1(u_1) \phi_2(u_2)
\]
is Lipschitz-continuous with respect to the $l_1$-norm, with constant $\gamma := \max \{c_2 \gamma_1, c_1 \gamma_2\}$. Such connections arise in the context of attention units in neural networks, which use  (bounded) activation maps such as $\tanh$.
\begin{figure}[h]
\centering
\includegraphics[height=.33\textheight]{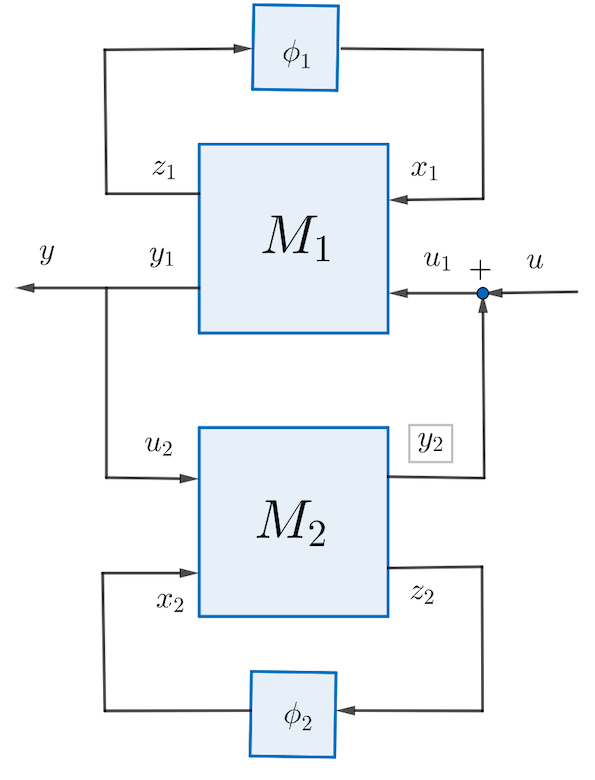}
\caption{Feedback connection of two implicit models.}
\label{fig:feedback}
\end{figure}

Finally, feedback connections are also possible. Consider two well-posed implicit systems:
\[
y_i = C_ix_i+D_iu_i , \;\; x_i = \phi_i(A_ix_i+Bu_i), \;\; i=1,2.
\]
Now let us connect them in a feedback connection: the combined system is described by the implicit rule~\cref{eq:pred-rule-top}, where $u_1 = u+y_2$, $u_2 = y_1 = y$. The feedback system is also an implicit model of the form~(\ref{eq:pred-rule-top}), with appropriate matrices $(A,B,C,D)$, and activation map acting blockwise: $\phi(z_1,z_2) = (\phi_1(z_1),\phi_2(z_2))$ and state $(x_1,x_2)$. In the simplified case when $D_1 = D_2 = 0$, the feedback connection has the model matrix
\[
\left(\begin{array}{c|c} 
A & B \\\hline C & D
\end{array}\right) = 
\left(\begin{array}{cc|c} 
A_1 & B_1C_2 & B_1 \\ B_2C_1 & A_2 & 0 \\\hline
C_1 & 0 & 0
\end{array}\right).
\]
Note that the connection is not necessarily well-posed when, even if $A_1 \in \text{WP}(\phi_1), A_2 \in \text{WP}(\phi_2) $.


\subsection{Scaling implicit models}
\label{sub:scaling}

Assume that the activation map $\phi$ is componentwise non-expansive (CONE) and positively homogeneous, as is the ReLU, or its leaky version. Consider an implicit model of the form~\cref{eq:pred-rule-top}, and assume it satisfies the PF sufficient condition for well-posedness of~\cref{thm:pf-wp}: $\lpf(|A|) \le \kappa$, where $0 \le \kappa <1$ is given. Then, there is another implicit model with the same activation map $\phi$, matrices $(A',B', C',D')$, which has the same prediction rule (i.e. $y'(u)=y(u), \forall u$), and satisfies $\|A'\|_\infty < 1$.  

This result is a direct consequence of the following expression of the PF eigenvalue as an optimally scaled $l_\infty$-norm, known as the Collatz-Wielandt formula, see~\cite[p. 666]{meyer2000matrix}:
\begin{equation}
    \label{eq:collatz-wiedlandt}
\lpf(|A|) = \inf_{S} \: \|SAS^{-1}\|_\infty ~:~ S = \diag(s), \;\; s>0.
\end{equation}
When the eigenvalue is simple, the optimal scaling vector is positive: $s>0$, and the new model matrices are obtained by diagonal scaling:
\begin{equation}\label{eq:scaling-idl}
    \left(\begin{array}{c|c} A' & B' \\\hline C ' & D' \end{array}\right) = 
    \left(\begin{array}{c|c} SAS^{-1} & SB \\\hline CS^{-1} & D' \end{array}\right),
\end{equation}
where $S = \diag(s)$, with $s>0$ a Perron-Frobenius eigenvector of $|A|$.

In a training problem, this result allows us to consider the convex constraint $\|A\|_\infty <1$ in lieu of its Perron-Frobenius eigenvalue counterpart. This result also allows us to rescale any given implicit model, such as one derived from deep neural networks, so that the norm condition is satisfied; we will exploit this in our robustness analyses~\cref{sec:robustness}.

\section{Implicit models of deep neural networks}
\label{sec:nn}
A large number of deep neural networks can be modeled as implicit models. Our goal here is to show how to build a well-posed implicit model for a given neural network, assuming the activation map satisfies the componentwise non-expansiveness (CONE), or the more general block Lipschitz-continuity (BLIP) condition, as detailed in~\ref{sub:ass-phi}. Thanks to the composition rules of~\cref{sub:composition}, it suffices to model individual layers, since a neural network is just a cascade connection of such layers. The block Lipschitz structure then emerges naturally as the result of composing the layers.

We will find that the resulting models always have a strictly (block) upper triangular matrix $A$, which automatically implies that these models are well-posed; in fact the equilibrium equation can be simply solved via backwards substitution. In turn, models with strictly upper triangular structure also naturally satisfy the PF sufficient condition for well-posedness: for example, in the case of a componentwise non-expansive map $\phi$, the matrix $|A|$ arising in~\cref{thm:pf-wp} is also strictly upper triangular, and therefore all of its eigenvalues are zero. A similar result also holds for the case when $\phi$ is block Lipschitz continuous, as defined in~\cref{sub:ass-phi}.

\subsection{Basic operations}
\paragraph{Activation at the output}
It is common to have an activation at the output level, as in the rule
\begin{equation}\label{eq:pred-rule-top-out}
  \hat{y}(u) = \phi_o(Cx+Du) , \;\; 
  x = \phi(Ax+Bu) ,
\end{equation}
with $\phi_o$ the output activation function.
We can represent this rule as in~\cref{eq:pred-rule-top}, by introducing a new state variable: with $\tilde{\phi} : = (\phi_o,\phi)$,
\[
\begin{pmatrix}
z \\ x
\end{pmatrix} = \tilde{\phi} \left( \begin{pmatrix} 
0 & C \\ 0 & A
\end{pmatrix} \begin{pmatrix}
z \\ x
\end{pmatrix} + \begin{pmatrix} 
D \\ B
\end{pmatrix}u
\right), \;\; 
  y = z . 
\]
The rule is well-posed with respect to $\tilde{\phi}$ if and only if the matrix $A$ is, with respect to $\phi$.

\paragraph{Bias terms and affine rules} The affine rule
\[
y = Cx+Du + d , \;\; x = \phi(Ax+Bu+b),
\]
where $d \in \reals^q$, $b \in \reals^n$, is handled by simply appending a $1$ at the end of the input vector $u$.

A related transformation is useful with activation functions $\phi$, such as the sigmoid, that do not satisfy $\phi(0) = 0$. Consider the rule
\[
y = Cx + Du, \;\; x = \phi(Ax+Bu) .
\]
Defining $\tilde{\phi}(\cdot) := \phi(\cdot)- \phi(0)$ and using the state vector $\tilde{x} := x + \phi(0)$, we may represent the above rule as
\[
y = C \tilde{x}  + \begin{pmatrix}
D & -C\phi(0)
\end{pmatrix} \begin{pmatrix}
u \\ 1
\end{pmatrix}, \;\; \tilde{x} = \tilde{\phi}(A\tilde{x}+\begin{pmatrix}
B & -A\phi(0)
\end{pmatrix} \begin{pmatrix}
u \\ 1
\end{pmatrix}) .
\]
Again, the rule is well-posed if and only if the matrix $A$ is.

\paragraph{Batch normalization} Batch normalization consists in including in the prediction rule a normalization step using some estimates of input mean $\hat{u}$ and variance ${\sigma}>0$ that are based on a batch of training data. The parameters $\hat{u},{\sigma}$ are given at test time. This step is a simple affine rule:
\[
y = Du + d, \mbox{ where } [D,d] := \diag(\sigma)^{-1}[I,-\hat{u}].
\]

\subsection{Fully connected feedforward neural networks} 
\label{sub:feedforward_neural_networks_are_a_special_case}
Consider the following prediction rule, with $L>1$ fully connected layers:
\begin{equation}\label{eq:standard-nn-pred-rule}
\hat{y}(u) = W_L x_L, \;\; x_{l+1} = \phi_l(W_l x_l), \;\; x_0 = u.
\end{equation}
Here $W_l \in \reals^{n_{l+1}\times n_l}$ and $\phi_l \:: \: \reals^{n_{l+1}} \rightarrow \reals^{n_{l+1}}$, $l=1,\ldots,L$, are given weight matrices and activation maps, respectively. We can express the above rule as~\cref{eq:pred-rule}, with $x = (x_L,\ldots,x_1)$, and
\begin{equation}\label{eq:standard-nn-pred-rule-implicit-model}
\left(\begin{array}{c|c} 
A & B \\\hline C & D
\end{array}\right)
= 
\left(\begin{array}{cccc|c} 
0 & W_{L-1} & \ldots & 0 & 0 \\
& 0 & \ddots & \vdots & \vdots \\
&& \ddots & W_1 & 0 \\
&& & 0 & W_0 \\\hline
W_{L} & 0 & \ldots &0 & 0
\end{array}\right) ,
\end{equation}
and with an appropriately defined blockwise activation function $\phi$, defined as operating on an  $n$-vector $z=(z_L,\ldots,z_1)$ as $\phi(z) = (\phi_L(z_L), \ldots,\phi_1(z_1))$. Due to the strictly upper triangular structure of $A$, the system is well-posed, irrespective of $A$; in fact the equilibrium equation $x = \phi(Ax+Bu)$ is easily solved via backward block substitution, which corresponds to a simple forward pass through the network. 

\subsection{Convolutional layers and max-pooling}
\label{sub:cnn}
A single convolutional layer can be represented as a linear map:
$y = Du$, where $u$ is the input, $D$ is a matrix that represents the (linear) convolution operator, with a ``constant-along-diagonals'', Toeplitz-like structure. For example a 2D convolution with a 2D kernel $K$ takes a $3 \times 3$ matrix $U$ and produces a $2 \times 2$ matrix $Y$. With
\[
U = \begin{pmatrix} u_1 & u_2 & u_3 \\ u_4 & u_5 & u_6 \\ u_7 & u_8 & u_9 \end{pmatrix}, \;\; K = \begin{pmatrix}
k_1 & k_2 \\ k_3 & k_4 \end{pmatrix} ,
\]
the convolution can be represented as $y=Du$, with $y,u$ vectors formed by stacking the rows of $U,Y$ together, and
\[
D = \begin{pmatrix} 
k_1 & k_2 & 0 & k_3 & k_4 & 0 & 0 & 0 & 0 \\
0 & k_1 & k_2 & 0 & k_3 & k_4 & 0 & 0 & 0 \\
0 & 0 & 0 & k_1 & k_2 & 0 & k_3 & k_4 & 0 \\
0 & 0 & 0 & 0 & k_1 & k_2 & 0 & k_3 & k_4 
\end{pmatrix}.
\]
\begin{figure}[h!]
\centering
\includegraphics[width=0.5\columnwidth]{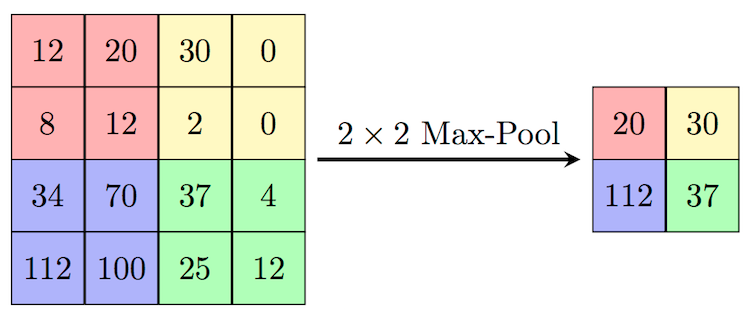}
\caption{\label{fig:max-pool} A max-pooling operation: the smaller image contains the maximal pixel values of each colored area.}
\end{figure}
Often, a convolutional layer is combined with a max-pooling operation. The latter forms a down-sample of an image, which is a  smaller image that contains the largest pixel values of specific sub-areas of the original image. Such an operation can be represented as 
\[
y_j = \max_{1 \le i \le h} \: (B_{j} u)_i, \;\; j \in [q].
\]
In the above, $p = qh$, and the matrices $B_{j} \in \reals^{h \times p}$, $j \in [q]$, are used to select specific sub-areas of the original image. In the example of~\cref{fig:max-pool}, the number of pixels selected in each area is $h = 4$, the output dimension is $q = 4$, that of the input is $p = qh = 16$; vectorizing images row by row:
\[
\begin{pmatrix} 
B_1 \\\hline B_2 \\\hline B_3 \\\hline B_4 
\end{pmatrix} = 
\diag(M,M) \in \reals^{16 \times 16}, \;\; M:=
\begin{pmatrix} 
I_2 & 0 & 0 & 0 \\
0 & 0 & I_2 & 0 \\
0 & I_2 & 0  & 0 \\
0 & 0 & 0 & I_2 
\end{pmatrix} \in \reals^{8 \times 8},
\]
where $I_2$ is the $2 \times 2$ identity matrix. 

Define the mapping $\phi \::\: \reals^n \rightarrow \reals^n$, where $n = p$, as follows. For a $p$-vector $z$ decomposed in $q$ blocks $(z_1,\ldots,z_q)$, we set
$\phi(z_1,\ldots,z_q) = (\dsp\max(z_1) , \ldots, \max(z_q), 0, \ldots, 0)$. (The padded zeroes are necessary in order to make sure the input and output dimensions of $\phi$ are the same). We obtain the implicit model
\[
y = C\phi(Bu) = Cx , \mbox{ where } x := \phi(Bu).
\]
Here $C$ is used to select the top $q$ elements, 
\[
C = \begin{pmatrix} 
I_q & 0 & \ldots & 0 
\end{pmatrix} , \;\; 
B := \begin{pmatrix} 
B_1^\top & \ldots & B_q^\top 
\end{pmatrix}^\top .
\]
The Lipschitz constant of the max-pooling activation map $\phi$, with respect to the $l_\infty$-norm, is $1$, hence it verifies the BLIP condition of~\cref{sub:ass-phi}.

\subsection{Residual nets}
A building block in residual networks involves the relationship illustrated in~\cref{fig:resnets}.
\begin{figure}
    \centering
    \begin{minipage}{0.5\textwidth}
        \centering
 \includegraphics[width=.4\columnwidth]{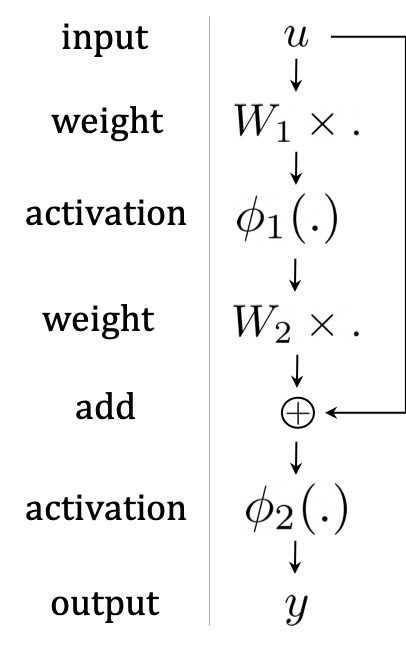}
\caption{\label{fig:resnets} Building block of residual networks.}
    \end{minipage}\hfill
    \begin{minipage}{0.5\textwidth}
        \centering
\includegraphics[width=1\columnwidth]{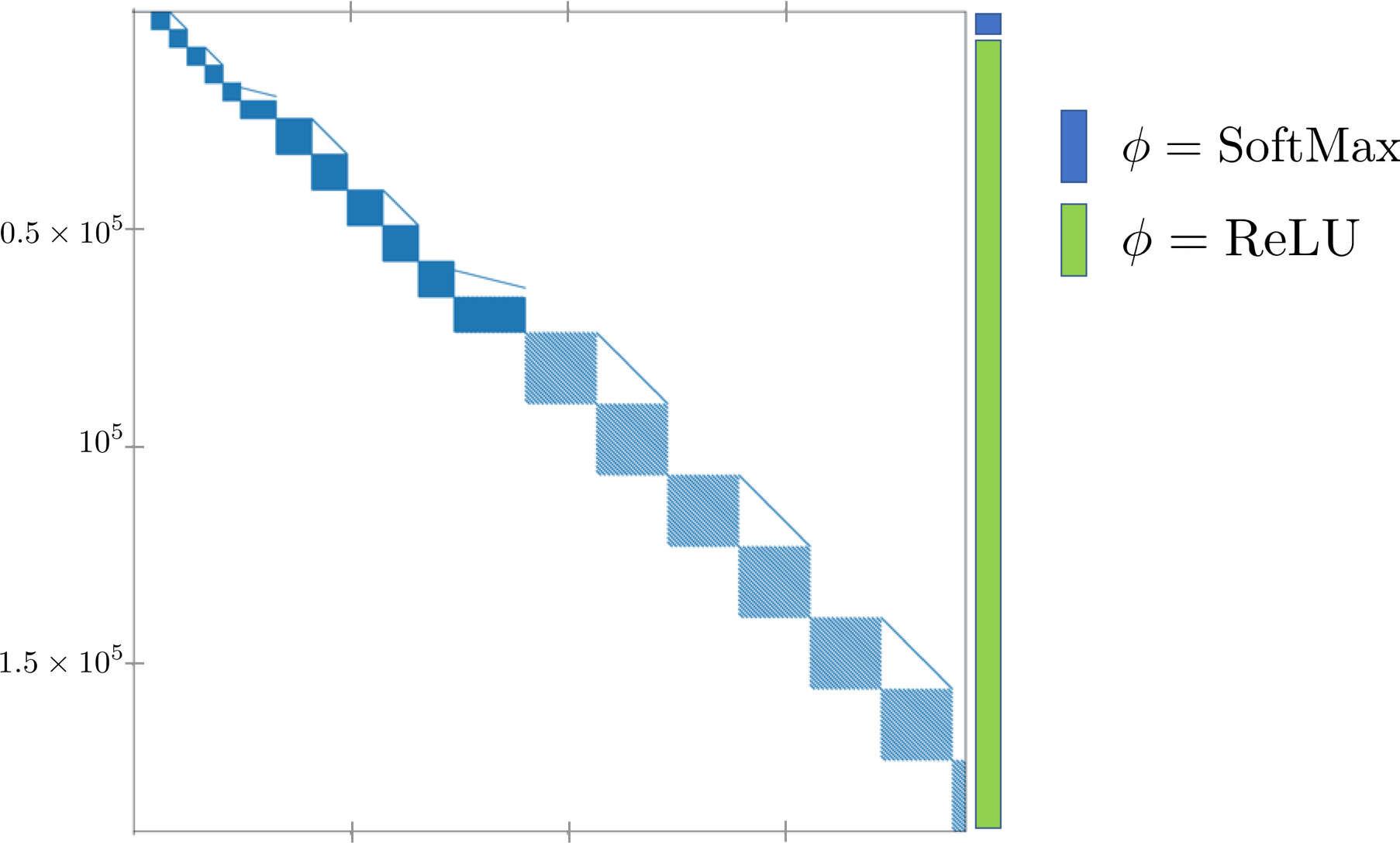}
\caption{\label{fig:resnet20} The matrix $A$ for a 20-layer residual network. Nonzero elements are colored in blue.}
    \end{minipage}
\end{figure}
Mathematically, a residual block combines two linear operations, with non-linearities in the middle and the end, and adds the input to the resulting output:
\[
y = \phi_2(u+W_2\phi_1(W_1u)),
\]
where $W_1,W_2$ are weight matrices of appropriate size.
The above is a special case of the implicit model~\cref{eq:pred-rule-top}: defining the blockwise map $\phi(z_2,z_1) = (\phi_2(z_2),\phi_1(z_1))$, 
\[
\begin{pmatrix} x_2 \\ x_1 \end{pmatrix} = \phi \left( 
\begin{pmatrix} 0 & W_2 \\ 0 & 0 \end{pmatrix}\begin{pmatrix} x_2 \\ x_1 \end{pmatrix} + \begin{pmatrix} I \\ W_1 \end{pmatrix}u \right), \;\;
y = x_2.
\]

\cref{fig:resnet20} displays the model matrix $A$ for a $20$-layer residual network. Convolutional layers appear as matrix blocks with Toeplitz (constant along diagonal) structure; residual unit correspond to the straight lines on top of the blocks. The network only uses the ReLU map, except for the last layer, which uses a softmax map.

\subsection{Recurrent units}
\label{sub:rnns}

Recurrent neural nets (RNNs) can be represented in an ``unrolled'' form as shown in~\cite{kawakami2008supervised}, which is the perspective we will consider here. The input to an RNN block is a sequence of vectors $\{u_1, \cdots, u_T \}$, where for every $t \in [T]$, $u_t \in \reals^p$. At each time step, the network takes in a input $u_t$ and the previous hidden state $h_{t-1}$ to produce the next hidden state $h_{t}$; the hidden state $h_t$ defines the state space or ``memory" of the network. The rule can be written as, 
    \begin{align}
        h_{t} = \phi_H(W_H h_{t-1} + W_I u_t), \;\; y_t = \phi_O(W_O h_{t}), \;\; t=1,\ldots,T,
        \label{rnn-hidden}
    \end{align}
Then, equations~\cref{rnn-hidden} can be expressed as an implicit model ~\cref{eq:pred-rule-top-out}, with $x = (h_{T}, \ldots, h_0)$, $u = (u_T, 
\dots, u_1)$, and weight matrices
    \begin{align}
        \left(
            \begin{array}{c|c}
                A & B \\ \hline
                C & 0 
            \end{array}
        \right)
        = 
        \left(
            \begin{array}{ccccc|cccc}
                0   &   W_H   & \cdots & 0        & 0       &   W_I  & \cdots & 0       & 0      \\
                0   &   0     & W_H    & \vdots   & \vdots  &   0    & W_I    & \vdots  & \vdots \\
                    &         & \ddots & \ddots   & 0       &        & \ddots & \ddots  & 0      \\
                    &         &        & 0        & W_H     &        &        & 0       & W_I    \\  
                \hline
                W_O & 0       & \cdots & 0        & 0       &   0    & \cdots & \cdots  & 0
            \end{array}
        \right)
    \end{align}
where $A, B$ are strictly upper block triangular and share the same block diagonal sub-matrices, $W_H$ and $W_I$ respectively, shared across all hidden states and inputs. 

Note that the above approach leads to matrices $A,B,C,D$ with a special structure; during training and test time, it is possible to exploit that structure in order to avoid ``unrolling'' the recurrent layers.


\subsection{Multiplicative units: LSTM, attention mechanisms and variants}
Some deep learning network architectures use multiplicative units. As mentioned in~\cref{sub:composition}, multiplication between variables in not Lipschitz-continuous in general, unless the inputs are bounded. Luckily, most of the multiplicative units used in practice involve bounded inputs.

In the case of Long Short-Term Memory (LSTM) \cite{yu2019review} or gated recurrent units (GRUs) \cite{cho2014learning}, a basic building block involves the \textit{product} of two input variables after each one passes through a bounded non-linearity. Precisely, the output takes the form
\[
y = \phi(u) := \phi_1(u_1) \phi_2(u_2) ,
\]
where $u_1,u_2 \in \reals$, and $\phi_1,\phi_2$ are both \textit{bounded} (scalar) non-linearities.  

Attention models \cite{bahdanau2014neural} use componentwise vector multiplication, usually involving a softmax operation~\cref{eq:soft-max}:
\[
y = \phi(u) = \phi_1(u_1) \odot \mbox{SoftMax}(u_2),
\]
where $\phi_1$ is a bounded, componentwise Lipschitz-continuous activation map, such as the sigmoid; and $W$ is a matrix of weights. The map $\phi$ is then Lipschitz-continuous with respect to the $l_1$-norm. As shown in Section \ref{sub:definition}, we can formulate these multiplicative units as well-posed implicit models.

\section{Robustness} 
\label{sec:robustness}
In this section, our goal is to analyze the robustness properties of a given implicit model~\cref{eq:pred-rule}. We seek to bound the state, output and loss function, under unknown-but-bounded inputs. This robustness analysis task is of interest in itself for diagnosis or for generating adversarial attacks. It will also inform our choice of appropriate penalties or constraints in the training problem. We assume that the activation map is Blockwise Lipschitz, and that the matrix $A$ of the implicit model satisfies the sufficient conditions for well-posedness outlined in~\cref{thm:pf-wp-block}.

\subsection{Input uncertainty models} 
\label{sub:input_uncertainty_model}
We assume that the input vector is uncertain, and only known to belong to a given set ${\cal U} \subseteq \reals^{p}$. Our results apply to a wide variety of sets ${\cal U}$; we will focus on the following two cases.

A first case corresponds to a box:
\begin{equation}\label{eq:U-def}
{\cal U}^{\rm box} := \left\{ u \in \reals^{p} ~:~ |u-u^0| \le \sigma_u \right\}.
\end{equation}
Here, the $p$-vector $\sigma_u>0$ is a measure of componentwise uncertainty affecting each data point, while $u^0$ corresponds to a vector of ``nominal'' inputs.  The following variant limits the number of changes in the vector $u$:
\begin{equation}\label{eq:U-def-card}
{\cal U}^{\rm card} := \left\{ u \in \reals^{p} ~:~ |u-u^0| \le \sigma_u, \;\; \Card(u-u^0) \le k \right\},
\end{equation}
where $\Card$ denotes the cardinality (number of non-zero components) in its vector argument, and $k <p$ is a given integer.

\subsection{Box bounds on the state vector} 
\label{sub:bounds_on_the_state_vector}
Assume first that $\phi$ is a CONE map, and that the input is only known to belong to the box set ${\cal U}^{\rm box}$~\cref{eq:U-def}. We seek to find componentwise bounds on the state vector $x$, of the form $|x-x^0| \le \sigma_x$, with $x$ and $x^0$ the unique solution to $\xi = \phi(A\xi+Bu)$ and $\xi = \phi(A\xi + Bu^0)$ respectively, and $\sigma_x >0$.
We have
\begin{align*}
|x-x^0| &= |\phi(Ax+Bu) - \phi(Ax^0+Bu^0)| \\
&\le|A(x-x^0) + B(u-u^0)| \le |A||x-x^0| + |B(u-u^0)| , 
\end{align*}
which shows in particular that
\[
\|x-x^0\|_\infty \le \|A\|_\infty \|x-x^0\|_\infty + \|B(u-u^0)\|_\infty,
\]
hence, provided $\|A\|_\infty <1$, we have:
\begin{equation}\label{eq:simple-x}
\|x-x^0\|_\infty \le \frac{\||B|\sigma_u\|_\infty}{1-\|A\|_\infty} . 
\end{equation}
With the cardinality constrained uncertainty set ${\cal U}^{\rm card}$~\cref{eq:U-def-card}, we obtain 
\begin{equation}\label{eq:simple-x-card}
\|x-x^0\|_\infty \le \frac{\delta}{1-\|A\|_\infty} , \;\;
\delta := \max_{1 \le i \le n} \: s_k(\sigma_u \odot |B|^\top e_i),
\end{equation}
with $e_i$ the $i$-th unit vector in $\reals^n$, and $s_k$ the sum of the top $k$ entries in a vector.

We may refine the analysis above, with a ``box'' (componentwise) bound. 
\begin{theorem}[Box bound on state, CONE map]\label{thm:box-bnd}
Assuming that $\phi$ is componentwise non-expansive, and $\lpf(|A|) < 1$. then, $I-|A|$ is invertible, and 
\begin{equation}\label{eq:box-x}
|x-x^0| \le \sigma_x : = (I-|A|)^{-1} |B|\sigma_u.
\end{equation}
\end{theorem}

\begin{proof}
See~\cref{app:box-state-bnds}.
\end{proof}
Note that the box bound can be computed via $(I-|A|)\sigma_x = |B| \sigma_u$, as the limit point of the fixed-point iteration
\[
\sigma_x(0) = 0, \;\; \sigma_x(t+1) = |A|\sigma_x(t) + |B|\sigma_u, \;\; t=0,1,2,\ldots,
\]
which converges since $\lpf(|A|) < 1$.


The case when $\phi$ is block Lipschitz (BLIP) involves the matrix of norms $N(A)$ defined in~\cref{eq:induced-norm-matrix}, as well as a related matrix defined for the input matrix $B$. Decomposing $B$ into blocks $B = (B_{li})_{l \in [L],i \in [p]}$, with $B_{li} \in \reals^{n_l}$, $l \in [L]$, we define the $L \times p$ matrix of norms
\begin{align}\label{B-blip-vecnorm}
N(B) := (\|B_{li}\|_{p_l})_{l \in [L], \: i \in [p]}.
\end{align}

\begin{theorem}[Box bound on the vector norms of the state, BLIP map]\label{thm:box-bnd-blip}
Assuming that $\phi$ is BLIP, and the corresponding sufficient well-posedness condition of~\cref{thm:pf-wp-block} is satisfied. Define $\Gamma := \diag(\gamma)$, where $\gamma$ is the vector of Lipschitz constants. Then, $I-\Gamma N(A)$ is invertible, and
\begin{equation}\label{eq:box-x-blip}
\eta(x-x^0) \le (I-\Gamma N(A))^{-1} \Gamma N(B)\sigma_u,
\end{equation}
where the vector of norms function $\eta(\cdot)$ is defined in~\cref{eq:vect-norms}.
\end{theorem}
\begin{proof}
See~\cref{app:box-state-bnds}.
\end{proof}

\subsection{Bounds on the output and the sensitivity matrix} 
\label{sub:bounds_on_the_output}


The above allows us to analyze the effect of effect of input noise on the output vector $y$. Let us assume that the function $\phi$ satisfies the CONE (componentwise non-expansiveness) condition~\cref{eq:cone-cond}. In addition, we assume that the stronger condition for well-posedness, $\|A\|_\infty <1$, is satisfied. (we can always rescale the model so as to ensure that property, provided $\lpf(|A|) <1$.)  For the implicit prediction rule~\cref{eq:pred-rule-top}, we then have
\[
\forall \: u, u^0 ~:~ \|\hat{y}(u)-\hat{y}(u^0)\|_\infty \le \kappa \|u-u^0\|_\infty,
\]
 where
\begin{equation}\label{eq:kappa-def}
\kappa := \frac{\|B\|_\infty  \|C\|_\infty}{1-\|A\|_\infty} + \|D\|_\infty.
\end{equation}
The above shows that the prediction rule is Lipschitz-continuous, with a constant bounded above by $\kappa$. This result motivates the use of the $\|\cdot\|_\infty$ norm as a penalty on model matrices $A,B,C,D$, for example via a convex penalty that bounds the Lipschitz constant,
\begin{equation}\label{eq:penalty-rob}
\kappa \le P(A,B,C,D) := \frac{1}{2}\frac{\|B\|_\infty^2 + \|C\|_\infty^2}{1-\|A\|_\infty} + \|D\|_\infty.
\end{equation}

We can refine this analysis with the following theorem.
\begin{theorem}[Box bound on the output, CONE map]\label{thm:box-bnd-output}
Assuming that $\phi$ is componentwise non-expansive, and $\lpf(|A|) < 1$. Then, $I-|A|$ is invertible, and
\begin{equation}\label{eq:box-y}
\forall \: u, u^0 ~:~ |\hat{y}(u)-\hat{y}(u^0)| \le S |u-u^0|,
\end{equation}
where the (non-negative) matrix
\[
S := |C|(I-|A|)^{-1}|B|+|D|
\]
is a ``sensitivity matrix'' of the implicit model with a CONE map.
\end{theorem} 

\begin{figure}[h]
\centering 
\includegraphics[height=.2\textheight]{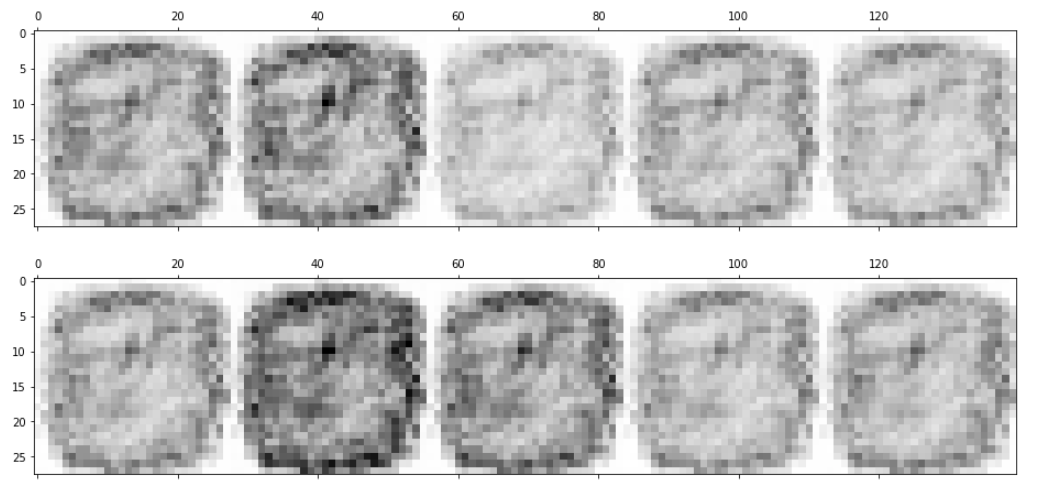}
\caption{Sensitivity matrix for a $3$-layer neural network with $q=10$ outputs and $n = 1094$ states.}
\label{fig:heatmap}
\end{figure}
The sensitivity matrix is a way to visualize the input-output properties of a given model. \cref{fig:heatmap} shows the sensitivity matrix for a given neural network for classification that has $q=10$ classes, with each row represented as an image. We observe that some classes have higher sensitivity values than others, which can be used to focus an attack on the model.   

Our refined analysis suggests a ``natural'' penalty to use during the training phase in order to improve robustness, namely $\|S\|_\infty$.

As before, the approach can be generalized to the case when $\phi$ is block Lipschitz (BLIP). Decomposing $C$ into column blocks $C = (C_1,\ldots,C_L)$, with $C_l \in \reals^{q \times n_l}$, $l \in [L]$, we define the matrix of (dual) norms
\[
N(C) := (\|C_{il}\|_{p_l^*})_{l \in [L], \: i \in [q]},
\]
where $p^*_l := 1/(1-1/p_l)$, $l \in [L]$.  Also recall the corresponding matrix of norms related to $B$~\cref{B-blip-vecnorm}.

\begin{theorem}[Box bound on the output, BLIP map]\label{thm:box-bnd-output-blip}
Assuming that $\phi$ is a BLIP map, and that the sufficient condition for well-posedness $\lpf(\Gamma N(A)) < 1$ is satisfied, with $\Gamma := \diag(\gamma)$ containing the Lipschitz constants associated with $\phi$. Then, $I-\Gamma N(A)$ is invertible, and
\begin{equation}\label{eq:box-y-blip}
\forall \: u, u^0 ~:~ |\hat{y}(u)-\hat{y}(u^0)| \le S |u-u^0|,
\end{equation}
where the (non-negative) $q \times p$ matrix
\[
S := N(C)(I-\Gamma N(A))^{-1}\Gamma N(B)+ |D|
\]
is a ``block sensitivity matrix'' of the implicit model with a BLIP map.
\end{theorem} 
\begin{proof}
See~\cref{app:box-state-bnds}.
\end{proof}

\subsection{Worst-case loss function} 
\label{sub:worst_case_loss_function}
In this section, we assume that the map $\phi$ satisfies the CONE property~\cref{eq:cone-cond}. We seek to find bounds on the loss function evaluated between a given ``target'' $y \in \reals^p$ and a prediction $\hat{y}$, using the ``box'' bounds on the state and output of~\cref{sub:bounds_on_the_state_vector} and~\cref{sub:bounds_on_the_output}.

Consider the case of squared Euclidean loss function, that is:
\[
{\cal L}(y,\hat{y}) = \|y-\hat{y}\|_2^2.
\]
The worst-case loss can be computed based on the box bound~\cref{eq:box-y}, as
\[
{\cal L}_{\rm wc}(y,\hat{y}^0) := \max_{|\hat{y}-\hat{y}^0| \le \sigma_y} \: {\cal L}(y,\hat{y}) = \||y-\hat{y}^0|+ \sigma_y \|_2^2.
\]

We can also consider the cross entropy loss function, that is: 
\[
{\cal L}(y,\hat{y}) =
\log(\sum_{i=1}^q e^{\hat{y}_i}) - y^\top \hat{y},
\]
where $y \ge 0$, $\ones^\top y = 1$ is given.

Assume that $D=0$ for simplicity, so that $\hat{y}^0 := Cx^0$ is the nominal output. Using the norm bound~\cref{eq:simple-x}, we have
$\|x-x^0\|_\infty \le \rho$, 
for an appropriate $\rho >0$. The corresponding worst-case loss is:
\begin{eqnarray*}
{\cal L}_{\rm wc}(y,\hat{y}^0) &:=& \max_{\delta x \::\: \|\delta x\|_\infty \le \rho} \: {\cal L}(y,\hat{y}^0 + C\delta x) \\
&=& \max_{\delta x \::\: \|\delta x\|_\infty \le \rho} \: \log(\sum_{i=1}^q e^{\hat{y}^0_i + c_i^\top \delta x}) - y^\top (\hat{y}^0+C\delta x) ,
\end{eqnarray*}
with $c_i^\top$ the $i$-th row of $C$. The above may be hard to compute, but we can work with a bound, based on evaluating the maximum above for each of the two terms independently. We obtain
\begin{eqnarray*}
{\cal L}_{\rm wc}(y,\hat{y}^0) &\le& \log(\sum_{i=1}^q e^{\hat{y}^0_i + \rho \|c_i\|_1}) - y^\top \hat{y}^0 + \rho \|C^\top y\|_1.
\end{eqnarray*}
Since $y \ge0 $ and $\ones^\top y = 1$, we have $\|C^\top y\|_1 \le \|C\|_\infty$, and we obtain the bound
\begin{eqnarray*}
{\cal L}(y,\hat{y}^0) \le {\cal L}_{\rm wc}(y,\hat{y}^0) &\le&  {\cal L}(y,\hat{y}^0) + 2 \rho \|C\|_\infty,
\end{eqnarray*}
which again involves the norm $\|C\|_\infty$ as a penalty. We can refine the analysis, based on the box bound~\cref{eq:box-y}:
\[
{\cal L}_{\rm wc}(y,\hat{y}^0) \le \log(\sum_{i=1}^q e^{\hat{y}^0_i+\sigma_{y,i}}) - y^\top \hat{y}^0 + y^\top \sigma_y.
\]

\subsection{LP relaxation for CONE maps} 
\label{sub:improved_state_bounds_based_on_lagrange_relaxation}
The previous bounds does not provide a direct way to generate an adversarial attack, that is, a feasible point $u \in {\cal U}$ that has maximum impact on the state. 
In some cases it may be possible to refine our previous box bounds via an LP relaxation, which has the advantage of suggesting a specific adversarial attack. Here we restrict our attention to the ReLU activation: $\phi(z)= \max(z,0)=z_+$, which is a CONE map. 

We consider the problem
\begin{equation}\label{eq:state-wc-l1}
p^* := \max_{x,u \in {\cal U}} \: \sum_{i \in [n]} f_i(x_i)  ~:~ x = z_+, \;\; z = Ax+Bu, \;\; |x-x^0| \le \sigma_x,
\end{equation}
where $f_i$'s are arbitrary functions. Setting $f_i(\xi) = (\xi-x_i^0)^2$, $i \in [n]$, leads to the problem of finding the largest discrepancy (measured in $l_2$-norm) between $x$ and $x^0$; setting $f_i(\xi) = -\xi$, $i \in [n]$, finds the minimum $l_1$-norm state.  By construction, our bound can only improve on the previous state bound, in the sense that
\[
p^* \le \sum_{i \in [n]} \max_{|\alpha|\le 1} \: f_i(x_i^0+\alpha  \sigma_{x,i}).
\]
Our result is expressed for general sets ${\cal U}$, based on the support function $\sigma_{\cal U}$ with values for $b \in \reals^p$ given by
\begin{equation}\label{eq:sU-def}
\sigma_{\cal U}(b) := \max_{u \in {\cal U}} \: b^\top u .
\end{equation}
Note that this function depends only on the convex hull of the set ${\cal U}$.
In the case of the box set ${\cal U}^{\rm box}$ given in \cref{eq:U-def}, there is a convenient closed-form expression:
\[
\sigma_{{\cal U}^{\rm box}}(b) = b^\top  u^0 + \sigma_u^\top |b| .
\]
Likewise for the cardinality-constrained set ${\cal U}^{\rm card}$ defined in~\cref{{eq:U-def-card}}, we have
\[
s_{{\cal U}_k}(b) = \max_{u \in {\cal U}^{\rm card}} \: b^\top u = b^\top u^0 + s_k(\sigma_u \odot |b|),
\]
where $s_k$ is the sum of the top $k$ entries of its vector argument---a convex function. 

The only coupling constraint in \cref{eq:state-wc-l1} is the affine equation, which suggests the following relaxation. 
\begin{theorem}[LP bound on the state]\label{thm:lp-bound}
An upper bound on the objective of problem~\cref{eq:state-wc-l1} is given by
\[
p^* \le \overline{p}: = \min_{\lambda} \: \sigma_{\cal U}(B^\top \lambda)+ \sum_{i \in [n]} g_i(\lambda_i,(A^\top \lambda)_i) ,
\]
where $s_{\cal U}$ is the support function defined in \cref{eq:sU-def},
and $g_i$, $i\in  [n]$, are the convex functions
\[
g_i ~:~ (\alpha ,\beta) \in \reals^2 \rightarrow g_i(\alpha,\beta) := \max_{\zeta \::\: |\zeta_+-x_i^0| \le \sigma_{x,i}} \: f_i(\zeta_+) -\alpha \zeta + \beta \zeta_+  , \;\; i \in [n] .
\]
If the functions $g_i$, $i \in [n]$ are closed, we have the bidual expression
\[
\overline{p}: = \max_{x, \: u \in {\cal U}} \: -\sum_{i \in [n]} g_i^*(-(Ax+Bu)_i,x_i),	
\]
where $g_i^*$ is the conjuguate of $g_i$, $i \in [n]$. 
\end{theorem}
\begin{proof}
See~\cref{app:proof_of_lp_bnds}.
\end{proof}

Consider for example $f_i(\xi) = c_i\xi$, $i \in [n]$, where $c \in \reals^n$ is given: the problem is
\[
p^* = \max_{x,\: u\in{\cal U}} \: c^\top x ~:~ x = z_+, \;\; z = Ax+Bu, \;\; |x-x^0| \le \sigma_x.
\]
Let $\underline{x} := (x^0-\sigma_x)_+$, $\overline{x} := x^0+\sigma_x$, so that the constraint $x \ge 0$, $|x-x^0| \le \sigma_x$ writes $\underline{x} \le x \le \overline{x}$. We have 
\begin{eqnarray*}
g_i(\alpha,\beta) &=& \max_{\zeta} \:  (\beta +c_i)\zeta_+ -\alpha \zeta ~:~ \underline{x}_i \le \zeta_+ \le \overline{x}_i \\
&=& \max_{\zeta_+,\zeta_-} \: \alpha (\zeta_--\zeta_+) +  (\beta +c_i)\zeta_+ ~:~ \zeta_\pm \ge0, \;\; \underline{x}_i \le \zeta_+ \le \overline{x}_i \\
&=& \max(\underline{x}_i(\beta-\alpha +c_i),\overline{x}_i(\beta-\alpha +c_i))  \mbox{ if $\alpha \le 0$, $+\infty$ otherwise.}
\end{eqnarray*}
This leads to the dual 
\[
p^* \le \overline{p}: = \min_{\lambda \ge 0} \: \sigma_{\cal U}(B^\top \lambda)+ \sum_{i \in [n]} \max(\underline{x}_i((A^\top \lambda)_i +c_i-\lambda_i) ,\overline{x}_i((A^\top \lambda)_i +c_i-\lambda_i) .
\]
For every $i \in [n]$, the function $g_i$ is closed, with conjugate given by
\begin{eqnarray*}
g_i^*(v,w) &=& \max_{\alpha \le 0, \: \beta} \: v \alpha + w \beta -\max(\underline{x}_i(\beta-\alpha +c_i),\overline{x}_i(\beta-\alpha +c_i)) \\
&=& \min_{\tau \::\: |\tau| \le 1} \: \max_{\alpha \le 0, \: \beta} \: v \alpha + w \beta -\tau\underline{x}_i(\beta-\alpha +c_i) - (1-\tau)\overline{x}_i(\beta-\alpha +c_i) \\
&=& w \mbox{ if } -v \le w \in [\underline{x}_i,\overline{x}_i], \;\; + \infty \mbox{ otherwise.}
\end{eqnarray*}
We obtain a ``natural'' relaxation:
\begin{eqnarray*}
p^* &\le& \overline{p} = \max_{x, \: u \in {\cal U}} \: c^\top x ~:~ x \ge Ax+Bu, \;\; x \ge 0, \;\; |x-x^0| \le \sigma_x.
\end{eqnarray*}

We consider a variant of the previous problem, where we seek a sparse adversarial attack: the cardinality of changes in the input is constrained in the set ${\cal U}^{\rm card}$. 
\[
p^* = \max_{x, \: u \in {\cal U}^{\rm card}} \: c^\top x ~:~ x = z_+, \;\; z = Ax+Bu, \;\; |x-x^0| \le \sigma_x,
\]
where ${\cal U}^{\rm card}$ is the set~\cref{eq:U-def-card}, with $k \in [p]$ given. The relaxation expresses as
\[
p^* \le \overline{p} :=  \min_{\lambda} \: \lambda^\top Bu^0 + s_k (\sigma_u \odot |B^\top \lambda|)+ \sum_{i \in [n]} \max(\underline{x}_i((A^\top \lambda)_i +c_i-\lambda_i) ,\overline{x}_i((A^\top \lambda)_i +c_i-\lambda_i) .
\]
The bidual writes
\begin{eqnarray*}
p^* &\le& \overline{p} = \max_{x, \: u} \: c^\top x ~:~ 
\begin{array}[t]{l}
x \ge Ax+Bu, \;\; x \ge 0, \;\; |x-x^0| \le \sigma_x, \\ 
\|\diag(\sigma_u)^{-1}  (u-u^0)\|_1 \le k, \;\;
|u-u^0| \le \sigma_u.
\end{array}
\end{eqnarray*}
where $\diag(\cdot)$ is a diagonal matrix formed with the entries of its vector argument. The above provides a heuristic for a sparse adversarial attack, based on any vector $u$ that is optimal for the above.

\subsection{SDP relaxations for maximal state discrepancy with ReLU activation} 
\label{sub:sdp_relaxations}

Assume again that $\phi$ is the ReLU, and consider the box uncertainty set ${\cal U}^{\rm box}$~\cref{eq:U-def}. Our goal is to find the largest state discrepancy, measured in the Euclidean norm:
\begin{align}\label{eq:ss-ncvx}
    p^* = \max_{x, \: u \in {\cal U}^{\rm box}} \: \|x-x^0\|_2^2 ~:~ x = \phi (Ax+Bu), \;\; |x-x^0| \le \sigma_x,
\end{align}
Based on a quadratic representation of $\phi$, we can arrive at a SDP relaxation of \eqref{eq:ss-ncvx}. We first use the fact that for any two $n$-dimensional vectors $x,z$:
\begin{equation}\label{eq:phi-rep-quad}
x = \phi(z) \Longleftrightarrow x \ge 0, \;\; x \ge z, \;\; x \odot x \le x \odot z,
\end{equation}
where $\odot$ is the Hadamard (component-wise) product. Next, we use the fact that any bound the form $l \leq z-z_0 \leq u$ where $z,l,u$ are vectors can be rewritten as
\begin{align}
    l \leq z-z_0 \leq u &\Longleftrightarrow (z-z_0-l) \odot (z-z_0-u) \leq 0 \\
    &\Longleftrightarrow z \odot z \leq (2z_0 + l + u) \odot z - (l + z_0) \odot (u + z_0)
\end{align}
These two reformulations allow us to represent \eqref{eq:ss-ncvx} as a non-convex QCQP:
\begin{align}\label{eq:state-eq-eucl}
    p^\ast = &\max_{x,u} \; \|x-x^0\|_2^2  \\
    &\text{s.t.} \;\;\;  x \geq 0, \;\; x \geq Ax + Bu, \;\; x \odot x \leq x \odot (Ax + Bu) \nonumber \\
    & \;\;\;\;\;\;\; x \odot x \leq (2x_0) \odot x - (-\sigma_x + x_0) \odot (\sigma_x + x_0) \nonumber\\
    & \;\;\;\;\;\;\; u \odot u \leq (2u_0) \odot u - (-\sigma_u + u_0) \odot (\sigma_u + u_0) \nonumber
\end{align}
We can derive an upper bound on the value of this QP by solving the Lagrange dual of the problem, leading to a semidefinite program. However, it is a well known result of quadratic optimization (for example see Appendix B of \cite{boyd2004convex}) that the bidual of a non-convex QCQP simply becomes the canonical rank relaxed version of \eqref{eq:ss-ncvx} when we redefine the objective in terms of a new variable $Z = zz^\top$, where $z = [x, \; u, \; 1]^\top$. Making this substitution and relaxing the equality constraint into a semidefinite constraint, the bidual of \eqref{eq:ss-ncvx} becomes
\begin{align}\label{eq:bidual}
p^{\ast\ast} = \max_{X,x,u} &\; \text{Tr}(X[1,1]) - 2x^\top x^0 + \|x^0\|_2^2  \\
&\text{s.t.} \;\;  \left[ \begin{array}{c|c}
X & \begin{array}{c}
x \\ 
u
\end{array} \\ \hline
\begin{array}{cc}
x^\top & u^\top
\end{array} & 1
\end{array}
\right] \succeq 0, \nonumber \\
& \;\; x \geq 0, \;\; x \geq Ax + Bu, \;\; \diag(X[1,1]) \leq \diag(AX[1,1]) + \diag(BX[2,1]), \nonumber\\
& \;\; \diag(X[2,2]) \leq (2u^0) \circ u - (-\sigma_u + u^0) \circ (\sigma_u + u^0), \nonumber\\
& \;\; \diag(X[1,1]) \leq (2x^0) \circ x - (-\sigma_x + x^0) \circ (\sigma_x + x^0), \nonumber
\end{align}
where used a rank relaxation involving
\begin{align*}
\begin{bmatrix}
x \\
u \\ 
1
\end{bmatrix}
\begin{bmatrix}
x \\
u \\ 
1
\end{bmatrix}^\top = \begin{bmatrix}
xx^\top & xu^\top & x \\
ux^\top & uu^\top & u \\
x^\top & u^\top & 1
\end{bmatrix} = \left[ \begin{array}{c|c}
X & \begin{array}{c}
x \\ 
u
\end{array} \\ \hline
\begin{array}{cc}
x^\top & u^\top
\end{array} & 1
\end{array}
\right]
\end{align*}
By construction, \eqref{eq:bidual} is convex. One immediate reason why we wish to solve the bidual as opposed to the Lagrange dual of \eqref{eq:ss-ncvx} is that we are able to get a candidate state $x$ and attack $u$, which can be viewed, respectively, as the perturbed state resulting from an adversarial attack on our network. This attack can be retrieved from the bidual in different ways. The first one would be to look at the optimal $u$ and the second would be to do a rank one decomposition of $X$ and extract $u$ from it. 

\section{Sparsity and Model Compression} 
\label{sec:sparsity_and_topology_optimization}
In this section, we examine the role of sparsity and low-rank structure in implicit deep learning, specifically in the model parameter matrix 
\[
M := \begin{pmatrix} A & B \\ C & D \end{pmatrix} .
\]
We assume that the activation map $\phi$ is a CONE map, as defined in~\cref{sub:ass-phi}. Most of our results can be generalized to BLIP maps.

Since a CONE map acts componentwise, the prediction rule~\cref{eq:pred-rule} is invariant under permutations of the state vector, in the sense that, for any $n \times n$ permutation matrix, the matrix $\diag(P,I)M\diag(P^\top ,I)$ represents the same prediction rule as $M$ given above. 

Various kinds of sparsity of $M$ can be encouraged in the training problem, with appropriate penalties.  For example, we can use penalties that encourage many elements in $M$ to be zero; the advantage of such ``element-wise'' sparsity is, of course, computational, since sparsity in matrices $A,B,C,D$ will allow for computational speedups at test time. Another interesting kind of sparsity is rank sparsity, which refers to the case when model matrices are low-rank.

Next, we examine the benefits of row- (or, column-) sparsity, which refers to the fact that entire rows (or, columns) of a matrix are zero. Note that column sparsity in a matrix $N$ can be encouraged with a penalty in the training problem, of the form
\[
{\cal P}(N) = \sum_{i} \|Ne_i\|_\alpha
\]
where $\alpha>1$. Row sparsity can be handled via ${\cal P}(N^\top )$.

\subsection{Deep feature selection} 
\label{sub:interpretability}
We may use the implicit model to select features. Any zero column in the matrix $(B^\top ,D^\top )^\top $ means that the corresponding element in an input vector does not play any role in the prediction rule. We may thus use a column-norm penalty in the training problem, in order to encourage such a sparsity pattern: 
\begin{equation}\label{eq:Bsparse-penalty}
	{\cal P}(B,D)  = \sum_{j=1}^p \left\|
	\begin{pmatrix} B \\ D \end{pmatrix} e_j\right\|_\alpha,
\end{equation}
with $\alpha > 1$.

\subsection{Dimension reduction via row- and column-sparsity}
\label{sub:dim_red}
Assume that the matrix $A$ is row-sparse. Without loss of generality, using permutation invariance, we can assume that $M$ writes
\[
M =  \begin{pmatrix} A_{11} & A_{12} & B_1 \\ 0 & 0 & B_2 \\
C_1 & C_2 & D\end{pmatrix} ,
\]
where $A_{11}$ is square of order $n_1 < n$. We can then decompose $x$ accordingly, as $x=(x_1,x_2)$ with $x_1 \in \reals^{n_1}$, and the above implies $x_2 = \phi(B_2u)$. The prediction rule for an input $u \in \reals^p$ then writes
\[
\hat{y}(u) = C_1 x_1 + Du, \;\; x_1 = \phi(A_{11} x_1 + A_{12}\phi(B_2u) +B_1 u).
\]
The rule only involves $x_1$ as a true hidden feature vector. In fact, the row sparsity of $A$ allows for a computational speedup, as we simply need to solve a fixed-point equation for the vector with reduced dimensions, $x_1$.

Further assume that $(A,B)$ is row-sparse. Again without loss of generality we may put $M$ in the above form, with $B_2 = 0$. Then the prediction rule can be written 
\[
\hat{y}(u) = C_1 x_1 + Du, \;\; x_1 = \phi(A_{11} x_1 +B_1 u).
\]
This means that the dimension of the state variable can be fully reduced, to $n_1 <n$. Thus, row sparsity of $(A,B)$ allows for a reduction in the dimension of the prediction rule.

Assume now that the matrix $A$ is column-sparse. Without loss of generality, using permutation invariance, we can assume that $M$ writes
\[
M = \begin{pmatrix} A_{11} & 0 & B_1 \\ A_{21} & 0 & B_2 \\ C_1 & C_2 & D \end{pmatrix} ,
\]
where $A_{11}$ is square of order $n_1 < n$. We can then decompose $x$ accordingly, as $x=(x_1,x_2)$ with $x_1 \in \reals^{n_1}$. The above implies that the prediction rule for an input $u \in \reals^p$ writes
\[
\hat{y}(u) = C_1 x_1 + C_2 x_2 + Du, \;\; x_1 = \phi(A_{11} x_1 + B_1 u), \;\; x_2 = \phi(A_{21}x_1 +B_2u).
\]
Thus, column-sparsity allows for a computational speedup, since $x_2$ can be directly expressed as closed-form function of $x_1$. 

Now assume that $(A^\top ,C^\top )^\top $ is column-sparse. Again without loss of generality we may put $M$ in the above form, with $C_2 = 0$. We obtain that the prediction rule does not need $x_2$ at all, so that the computation of the latter vector can be entirely avoided. This means that the dimension of the state variable can be fully reduced, to $n_1 <n$. Thus, column sparsity of $(A^\top ,C^\top )^\top $ allows for a reduction in the dimension of the prediction rule.

To summarize, row or column sparsity of $A$ allows for a computational speedup; if the corresponding rows of $B$ (resp.\ columns of $C$) are zero, then the prediction rule involves only a vector of reduced dimensions.

\subsection{Rank sparsity} 
\label{sub:rank_sparsity}
Assume that the matrix $A$ is rank $k \ll n$, and that a corresponding factorization is known: $A = LR^\top $, with $L,R \in \reals^{n \times k}$.
In this case, for any $p$-vector $u$, the equilibrium equation $x = \phi(Ax+Bu)$ can be written as $x = \phi(Lz+Bu)$, where $z := R^\top x$. Hence, we can obtain a prediction for a given input $u$ via the solution of a low-dimensional fixed-point equation in $z \in \reals^k$:
\[
z = R^\top  \phi (Lz+Bu).
\]
It can be shown that, when $\phi$ is a CONE map, the above rule is well-posed if $\lpf(|R|^T|L|) <1$. Once a solution $z$ is found, we simply set the prediction to be $\hat{y}(u) = C\phi(Lz+Bu) +Du$.

At test time, if we use fixed-point iterations to obtain our predictions, then the computational savings brought about by the low-rank representation of $A$ can be substantial, with a per-iteration cost going from $O(n^2)$, to $O(kn)$ if we use the above.


\subsection{Model error analysis}
The above suggests to replace a given model matrix $A$ with a low-rank or sparse approximation, denoted $A^0$. The resulting state error can be then bounded, as follows. Assume that $|A-A^0| \le E$, where $E \ge 0$ is a known upper bound on the componentwise error in $A$. The following theorem, proved in \cref{app:proof_pertA}, provides \emph{relative} error bounds on the state, provided the perturbed system satisfies the well-posedness condition $\lpf(|A|) <1$. As before, we denote by $x^0,x$ the (unique) solutions to the unperturbed and perturbed equilibrium equations $\xi = \phi(A^0\xi +Bu)$ and $\xi = \phi(A\xi +Bu)$, respectively. 
\begin{theorem}[Effect of errors in $A$]\label{thm:pf-pertA}
Assuming that $\phi$ is a CONE map, and that $\lpf(|A^0+E|) < 1$. Then:
\begin{equation}\label{eq:pertA}
|x-x^0| \le (I-(|A^0+E|))^{-1}E x_0.
\end{equation}
\end{theorem}
\begin{proof}
	See \cref{app:proof_pertA}.
\end{proof}



\section{Training Implicit Models} 
\label{sec:training_problem}

\subsection{Setup} 
\label{sub:training_problem}
We are now given an input data matrix $U = [u_1,\ldots,u_m] \in \reals^{p \times m}$ and response matrix $Y = [y_1, \ldots, y_m] \in \reals^{q \times m}$, and seek to fit a model of the form~\cref{eq:pred-rule}, with $A$ well-posed with respect to $\phi$, which we assume to be a BLIP map.  We note that the rule~\cref{eq:pred-rule}, when applied to a collection of inputs $(u_i)_{1 \le i \le m}$, can be written in matrix form, as
\[
\hat{Y}(U) = CX+DU , \mbox{ where } X = \phi(AX+BU).
\]
where $U = [u_1,\ldots,u_m] \in \reals^{p \times m}$, and $\hat{Y}(U) = [\hat{y}(u_1), \ldots, \hat{y}(u_m)] \in \reals^{q \times m}$. 

We consider a training problem of the form
\begin{equation}\label{eq:training-pb}
\min_{A,B,C,D,X} \: {\cal L}(Y,CX+DU) + {\cal P} (A,B,C,D) ~:~ X = \phi(AX+BU),  \;\; A \in \WP.
\end{equation}
In the above, ${\cal L}$ is a loss function, assumed to be convex in its second argument, and ${\cal P}$ is a convex penalty function, which can be used to enforce a given (linear) structure (such as, $A$ strictly upper block triangular) on the parameters, and/or encourage their sparsity. 

In practice, we replace the well-posedness condition by the sufficient condition of \cref{thm:pf-wp-block}. As argued in~\cref{sub:scaling}, the latter can be further replaced without loss of generality with the easier-to-handle (convex) constraint; in the case of CONE maps, this condition is $\|A\|_\infty \le \kappa$, where $\kappa \in (0,1)$ is a hyper-parameter.


\paragraph{Examples of loss functions} 
\label{par:example_of_a_loss_function_cross_entropy_with_softmax}
For regression tasks, we may use the squared Euclidean loss: for $Y,\hat{Y} \in \reals^{q \times m}$,
\[
{\cal L}(Y,\hat{Y}) := \frac{1}{2} \|Y-\hat{Y}\|_F^2.
\]
For multi-class classification, a popular loss is a combination of negative cross-entropy with the soft-max: for two $q$-vectors $y,\hat{y}$, with $y \ge 0$, $y^\top \ones = 1$, we define
\[
{\cal L}(y,\hat{y})  = -y^\top \log \left( 
\frac{e^{\hat{y}}}{{\sum_{i=1}^q e^{\hat{y}_i}}} 
\right) = 
\log(\sum_{i=1}^q e^{\hat{y}_i}) - y^\top \hat{y}.
\]
We can extend the definition to matrices, by summing the contribution to all columns, each corresponding to a data point: for $Y,Z \in \reals^{q \times m}$,
\begin{equation}\label{eq:def-loss-ex}
{\cal L}(Y,\hat{Y})  = \sum_{j=1}^m \log\left(\sum_{i=1}^q e^{\hat{Y}_{ij}} \right) - \sum_{j=1}^m \sum_{i=1}^q Y_{ij}\hat{Y}_{ij} = \log(\ones^\top  \exp(\hat{Y}))\ones - \Tr Y^\top \hat{Y},
\end{equation}
where both the $\log$ and the exponential functions apply componentwise.

\paragraph{Examples of penalty functions} 
\label{par:examples_of_penalty_functions}
Via an appropriate definition of ${\cal P}$, we can make sure that the model matrix $A$ is well-posed with respect to $\phi$, either imposing an upper triangular structure for $A$, or via an $l_\infty$-norm constraint $\|A\|_\infty < 1$, or a Perron-Frobenius eigenvalue constraint. Note that in the case of the CONE maps such as the ReLU, due to scale invariance seen in~\cref{sub:scaling}, we can always replace a Perron-Frobenius eigenvalue constraint with a $l_\infty$-norm constraint.  

Beyond well-posedness, the penalty can be used to encourage desired properties of the model. For robustness, the convex penalty~\cref{eq:penalty-rob} can be used, provided we also enforce $\|A\|_\infty <1$. We may also use an $l_\infty$-norm bound on the sensitivity matrices $S$ appearing in~\cref{thm:box-bnd-output} or~\cref{thm:box-bnd-output-blip}. For feature selection, an appropriate penalty may involve the block norms~\cref{eq:Bsparse-penalty}; sparsity of the model matrices can be similarly handled with ordinary $l_1$-norm penalties on the elements of the model matrix $M=(A,B,C,D)$.

\paragraph{Fenchel divergence formulation} 
\label{par:fenchel_divergence_formulation}
Problem~(\ref{eq:training-pb}) can be equivalently written
\[
\min_{A,B,C,D,X} \: {\cal L}(Y,CX+DU) + {\cal P} (A,B,C,D) ~:~ F_\phi(X,AX+BU) \le 0, \;\; A \in \WP,
\]
where $F_\phi$ is the so-called Fenchel divergence adapted to $\phi$~\cite{gu2018fenchel}, applied column-wise to matrix inputs. In the case of the ReLU activation, for two given vectors $x,z$ of the same size, we have 
\begin{equation}\label{eq:relu-div}
F_\phi(x,z) := \frac{1}{2} x\odot x + \frac{1}{2} z_+ \odot z_+ - x \odot z \mbox{ if } x \ge 0,
\end{equation}
with $\odot$ the componentwise multiplication. We can then define $F_\phi(X,Z)$ with $X,Z$ two matrix inputs having the same number of columns, by summing over these columns. As seen in~\cite{gu2018fenchel}, a large number of popular activation maps can be expressed similarly.

We may also consider a relaxed form:
\[
\min_{A,B,C,D,X} \: {\cal L}(Y,CX+DU) + {\cal P} (A,B,C,D) + \lambda^\top F_\phi(X,AX+BU) , \;\; A \in \WP,
\]
where $\lambda >0$ is a $n$-vector parameter.


\subsection{Gradient Methods}
\label{sub:gradient-descent}
Assuming that $\phi$ is a CONE map for simplicity, 
we consider the problem
\begin{equation}\label{eq:training-pb-linf}
\min_{A,B,C,D,X} \: {\cal L}(Y,CX+DU) + {\cal P} (A,B,C,D) ~:~ X = \phi(AX+BU),  \;\; \|A\|_\infty \le \kappa,
\end{equation}
where $\kappa <1$ is given.

We can solve~\cref{eq:training-pb-linf} using (stochastic) projected gradient descent. The basic idea is to update the model parameters $M=(A,B,C,D)$ using stochastic gradients, by differentiating through the equilibrium equation.  It turns out that this differentiation requires solving an equilibrium equations involving a matrix variable, which can be very efficiently solved via fixed-point iterations. More precisely, the main effort in computing the gradient for a mini-batch of size say $b$ consists in solving matrix equations in a $n \times b$ matrix $V$:
\[
    V = \Psi \circ \left(A^\top V + C^\top L\right),
\]
where the columns of $L$ contains gradients of the loss with respect to $\hat{y}$, and $\Psi$ is a matrix whose columns contain the derivatives of the activation map, both evaluated at a specific training point. Due to the fact taht $A$ satisfies the PF sufficient condition for well-posedness with respect to $\phi$, the equation above has a unique solution; the matrix $V$ can be computed as the limit point of the convergent fixed-point iterations
\begin{equation}
    \label{eq:fixed-point-v}
V(t+1) = \Psi \circ \left(A^\top V(t) + C^\top L\right)
, \;\; t=0,1,2,\ldots
\end{equation}
The proof of the following result is given in \cref{app:grad-descent}. 

\begin{theorem} [Computing gradients] \label{thm:grad_exist}
Consider an implicit model ($\phi$, $A,B,C,D$), with $\phi$ a differentiable BLIP map. If $A$ is well-posed with respect to $\phi$, the gradients with respect to the model parameter matrices $A, B, C$ and $D$ exist, and are uniquely given via the solution of a fixed-point equation, as detailed in~\cref{app:grad-descent}.
\end{theorem}

In order to handle the well-posedness constraint $\|A\|_\infty \le \kappa$, the projected gradient method requires a projection at each step. This step corresponds to a sub-problem of the form:
\begin{equation}
    \label{eq:proj_nrm_ball}
\min_{A} \: \|A-A^0\|_F ~:~ \|A\|_\infty \le \kappa,
\end{equation}
with matrix $A^0$ given. The above problem is decomposable across rows, and can be efficiently solved using (vectorized) bisection, as detailed in~\cref{app:vec-bisec-proj}.

In some cases, it is desirable to obtain a (block) sparse model matrix $M$. One way to efficiently achieve sparsity is the use of a conditional gradient  (Frank-Wolfe) algorithm that typically only adds a few non-zeros elements to the model matrix $M$ at each step. When updating $A$ for example, considering $B,C,D$ fixed for simplicity, the method requires solving a problem of the form
\[
\max_{A} \: \Tr G^\top A ~:~ \|A\|_\infty \le \kappa ,
\]
where $G$ is a gradient matrix. The problem can be solved in closed-form.

\subsection{Block-coordinate descent methods} 
\label{sub:bcd_algorithm}
Block-coordinate descent methods use cyclic updates, optimizing one matrix variable at a time, fixing all the other variables. These methods are particularly interesting when the updates require solving convex problems. For instance, considering the training problem~\cref{eq:training-pb-linf}, given $X$, the problem is convex in the model matrix $M$. Then, given the model matrix $M$, finding $X$ consists in a feasibility problem that can be solved with fixed-point iterations. 

In the case of the Fenchel divergence problem relaxation, all the updates involve solving convex problems, as shown in \cite{gu2018fenchel}, as the Fenchel divergence $F_\phi$, for most of activation maps, is bi-convex in its two arguments---meaning that given the first argument fixed, it is convex in the second and vice-versa. We refer to \cite{gu2018fenchel, travacca2020} for more on Fenchel divergences in the context of implicit deep learning and neural networks. We  illustrate and give more detailed examples of such methods in the next section.


\section{Numerical experiments}
\label{sec:num}

\subsection{Learning real nonlinear functions via regression}
We start by illustrating the ability of the gradient method, as presented in \cref{sub:gradient-descent}, to learn the parameters of the implicit model, with  well-posedness enforced via a max-row-sum norm constraint, $\|A\|_\infty \leq 0.5$. We aim at learning a real function, as an example we focus on
$$f(u) = 5 \cos(\pi u) \exp(-\frac{1}{2}|u|),$$
\begin{figure}[ht]
\label{fig: nonlinear function}
\centering
\includegraphics[scale=0.2]{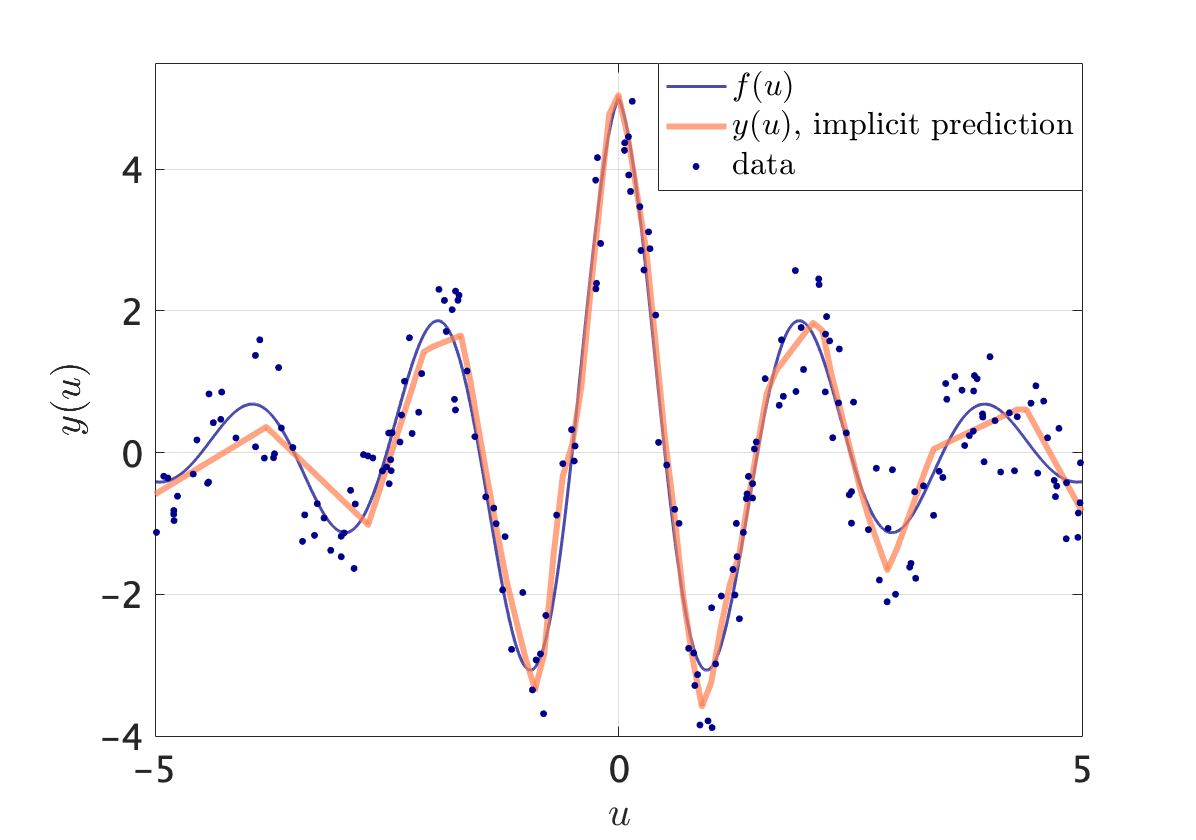}
\caption{Implicit prediction $y(u)$ comparison with $f(u)$}
\label{fig: activation}
\end{figure}
We select the input $u_i$, $i \in \{1, \cdots, m\}$ uniformly at random between $-5$ and $5$ with $m=200$; we add a random noise to the output, $y(u)=f(u)+w$ with $w$ taken uniformly at random between $-1$ and $1$, hence the standard deviation for $y(u)$ is $1 \slash \sqrt(3) \simeq 0.57$. We consider an implicit model of order $n = 75 $. We learn the parameters of the model by doing only two successive block updates: first, we update $(A, B)$ using stochastic projected gradient descent, the gradient being obtained with the implicit chain rule described in~\cref{app:grad-descent}. The RMSE across iterations for this block-update is shown in \cref{fig: activation2}. After this first update we achieve a RMSE of $1.77$. Second, we update $(C,D)$ using linear regression. After this update we achieve a RMSE of $0.56$. For comparison purposes, we also train a neural network with $3$ hidden layers of width $n/3 = 25$ using ADAM, mini-batches, and a tuned learning rate. We run Adam until convergence. We get a RMSE $ =0.65$, which is slightly above that of our implicit model. This first simple experiment shows the ability to fit nonlinear functions with implicit models as illustrated in~\cref{fig: activation}.
\begin{figure}[ht]
\centering
\includegraphics[scale=0.18]{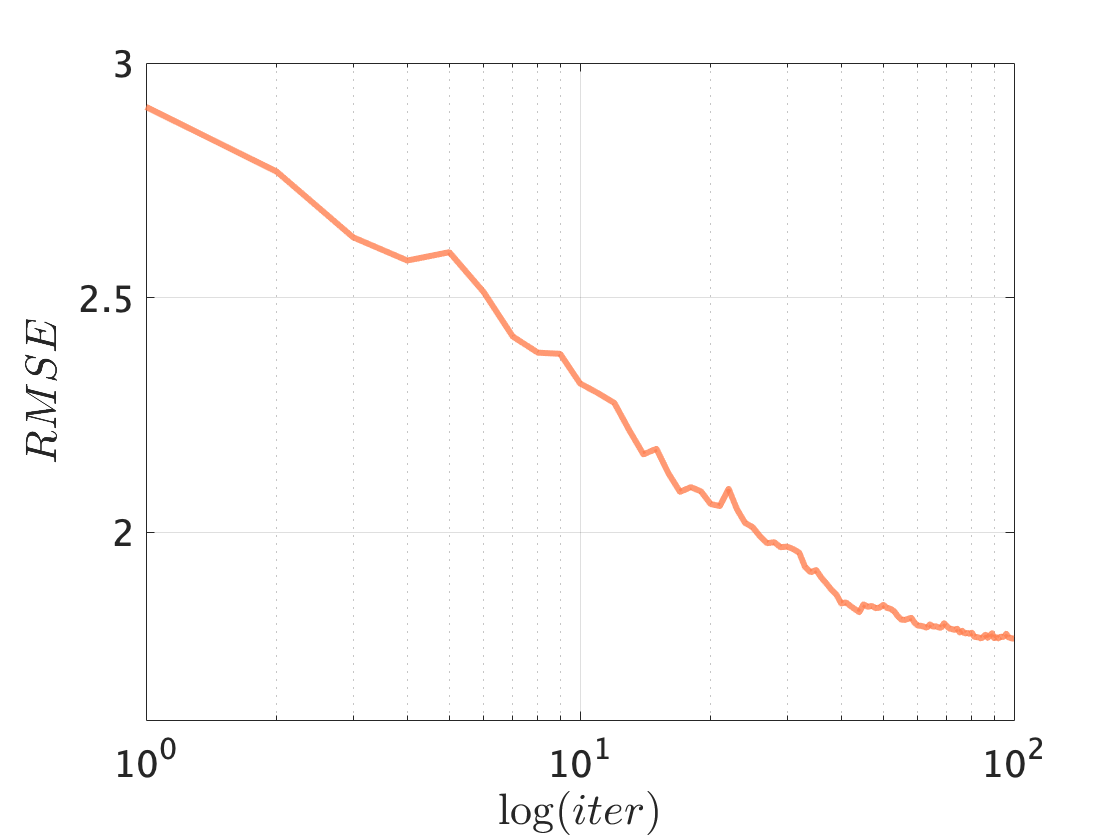}
\caption{RMSE across projected gradient iterations for the $(A, B)$ block update}
\label{fig: activation2}
\end{figure}

\subsection{Comparison with neural networks}
In that section we compare the performance of implicit models with that of neural networks. Experiments on both synthetic datasets and real datasets are conducted. The experiment results show that implicit model has the potential of matching or exceeding the performance of neural networks in various settings. To simplify the experiments, we do not apply any specific network structure or regularization during training. When it comes to training neural networks we will always use ADAM with a grid-search tuned step-size. Similarly, we will always use stochastic projected gradient descent as detailed in~\cref{app:grad-descent}. The choice of the number of hidden features $n$ for implicit models is always aligned with the neural network architecture for fair comparison as explained in~\cref{sub:feedforward_neural_networks_are_a_special_case}.

\subsubsection{Synthetic datasets}
We first consider two synthetic classification datasets: one generated from a neural network and an other from an implicit model. For each dataset, we then aim at fitting both an implicit model and a neural network. Each data point in the datasets contains an input $u\in \mathbb{R}^5$ ($p=5$) and output $y \in \mathbb{R}^2$ ($q=2$). The dataset is generated as follows.

We chose the input datapoints $u_i\in \reals^5$ i.i.d. by sampling each entry independently uniformly between $-1$ and $1$. The output data $y_i \in \reals^2$ is the one-hot representation of the argmax of the output from the generating model, whose parameters are chosen uniformly at random between $-1$ and $1$.

\begin{itemize}
    \item  For a neural network generating model, we us a fully connected 3-layer (5-5-5-2) feedforward neural network with ReLU activation.
    
    \item For an implicit model, we consider an implicit model with $n=10$ to match the number of hidden nodes in the neural network model. To ensure well-posedness, after sampling matrix the $A$, we proceed with a projection on the unit-sized infinity norm ball. (See \cref{app:vec-bisec-proj}).
\end{itemize}

For each dataset, we consider 20 training and test datapoints. The size of datasets are kept small to maintain the over-parameterized settings where the number of parameters is larger than the number of data points. Deep learning falls under the such regime \cite{belkin2019reconciling}. We run 5 separate experiments for both the neural network and implicit model generated dataset with the training model having the same hyperparameters and initialization as the generating models.
 
\vspace{3mm}

\begin{figure}[htbp]
\centering
\begin{minipage}[t]{0.48\textwidth}
\centering
\includegraphics[width=1\columnwidth]{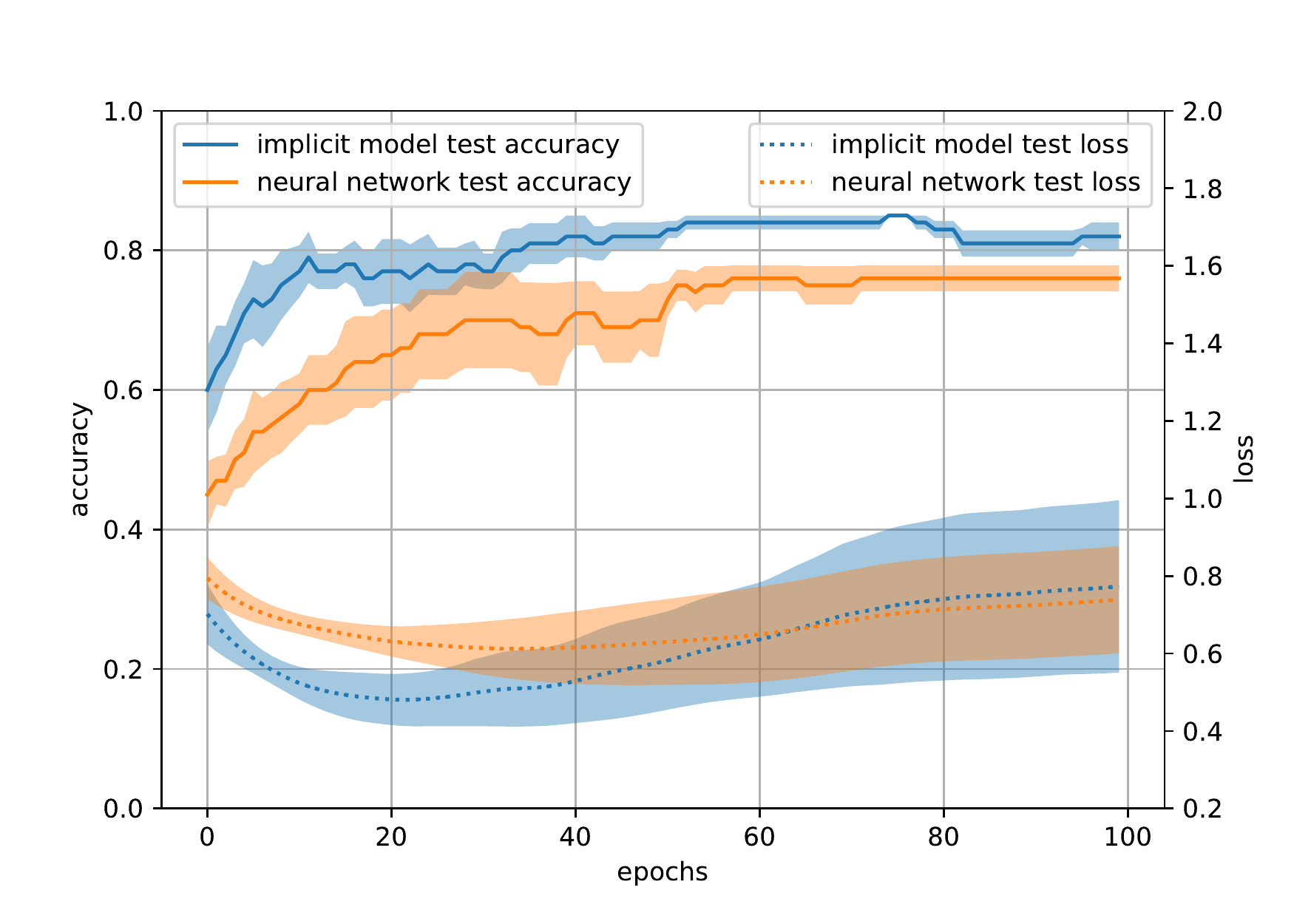}
\caption{Performance comparison on a synthetic dataset generated from a neural network. Average best accuracy, implicit: 0.85, neural network: 0.76. The curves are generated from 5 the different runs with the lines marked as mean and region marked as the standard deviation}
\label{fig:plotsynnn}
\end{minipage}
\hfill
\begin{minipage}[t]{0.48\textwidth}
\centering
\includegraphics[width=1\columnwidth]{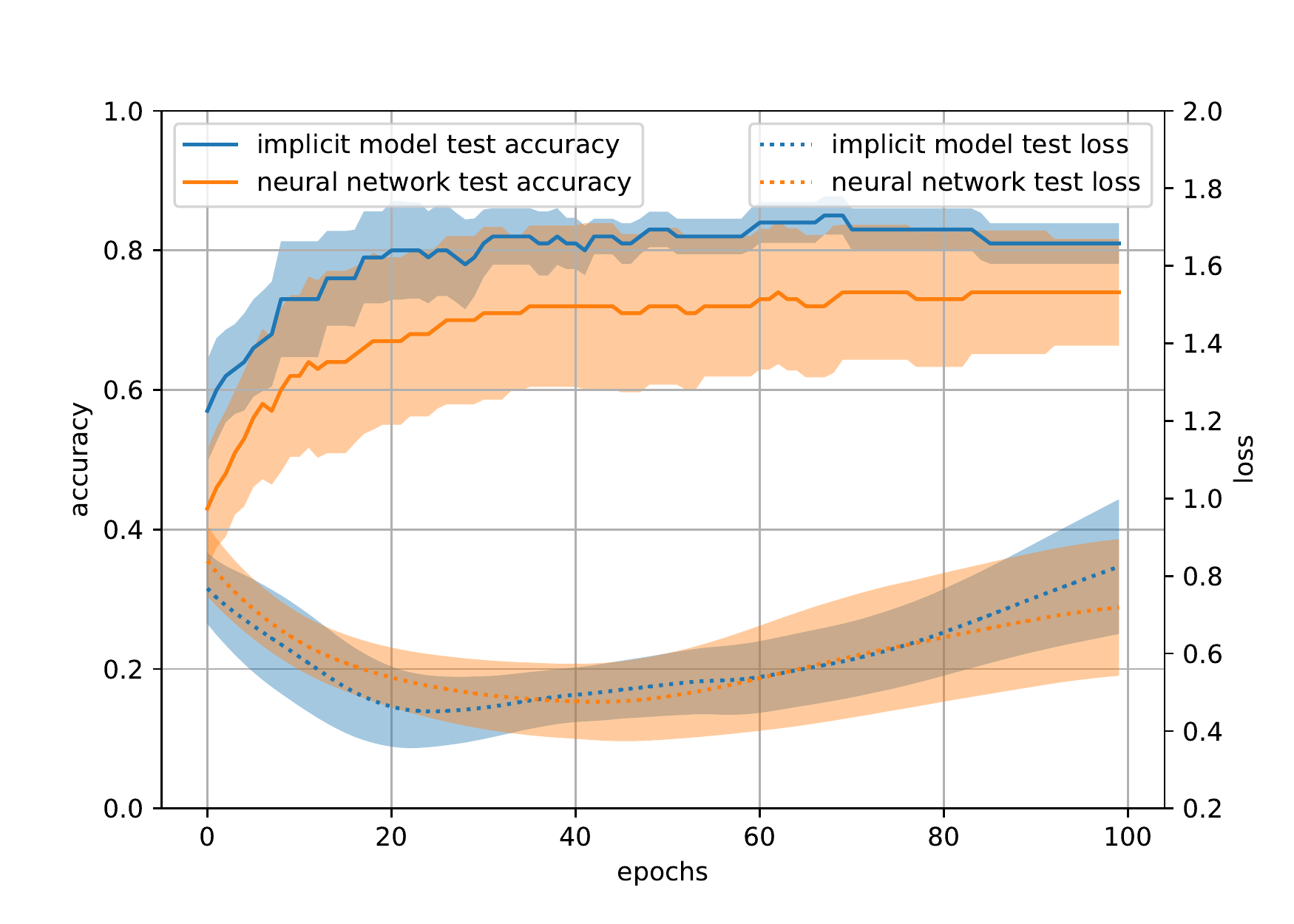}
\caption{Performance comparison on a synthetic dataset generated from an implicit model. Average best accuracy, implicit: 0.85, neural networks: 0.74. The curves are generated from 5 different runs with the lines marked as mean and region marked as the standard deviation over the runs.}
\label{fig:plotsynim}
\end{minipage}

\end{figure}

We find that the implicit model outperforms neural networks in both synthetic experiments. This may be explained by the increased modeling capacity of the implicit model, given similar parameter size, with respect to its neural network counterpart as mentioned in~\cref{sub:implicit_prediction_rules}.

\subsubsection{MNIST dataset}
\label{subsub: MNIST dataset}
In the digit classification dataset MNIST, the input data points are $28 \times 28$ pixels images of hand written digits, the output is the corresponding digit label. For training purposes, each image is reshaped into 784 dimensional vectors and normalized before training. There are $m=5 \times 10^4$ training data points and $10^4$ testing data points.  

The architecture we use for the neural network as a reference is a three-layer feedforward neural network (784-60-40-10) with ReLU activation. For the implicit model, we set $n=100$. We choose a batchsize of $100$ for both algorithms. 
The accuracy with respect to iterations is shown in \cref{fig:plotmnist}. We observe that the accuracy of the implicit model matches that of its neural network counterpart. 

\subsubsection{German Traffic Sign Recognition Benchmark}
\label{subsub: GTSRB dataset}
 In the German Traffic Sign Recognition Benchmark (GTSRB) \cite{stallkamp2011german}, the input data points are $32\times 32$ images with rgb channels of traffic signs consisting of 43 classes. Each image is turned into gray-scale before being reshaped into a 1024 dimensional vector and re-scaled to be between 0 and 1. There are 34,799 training data points and 12,630 testing data points. 

For reference, we use a (1024-300-100-43) feedforward neural network with ReLU activations. We choose $n=400$ for the implicit model. We choose a batchsize of 100. The accuracy with respect to iterations is shown in \cref{fig:plottraffic}.

\begin{figure}[htbp]
\centering
\begin{minipage}[t]{0.48\textwidth}
\centering
\includegraphics[width=1\columnwidth]{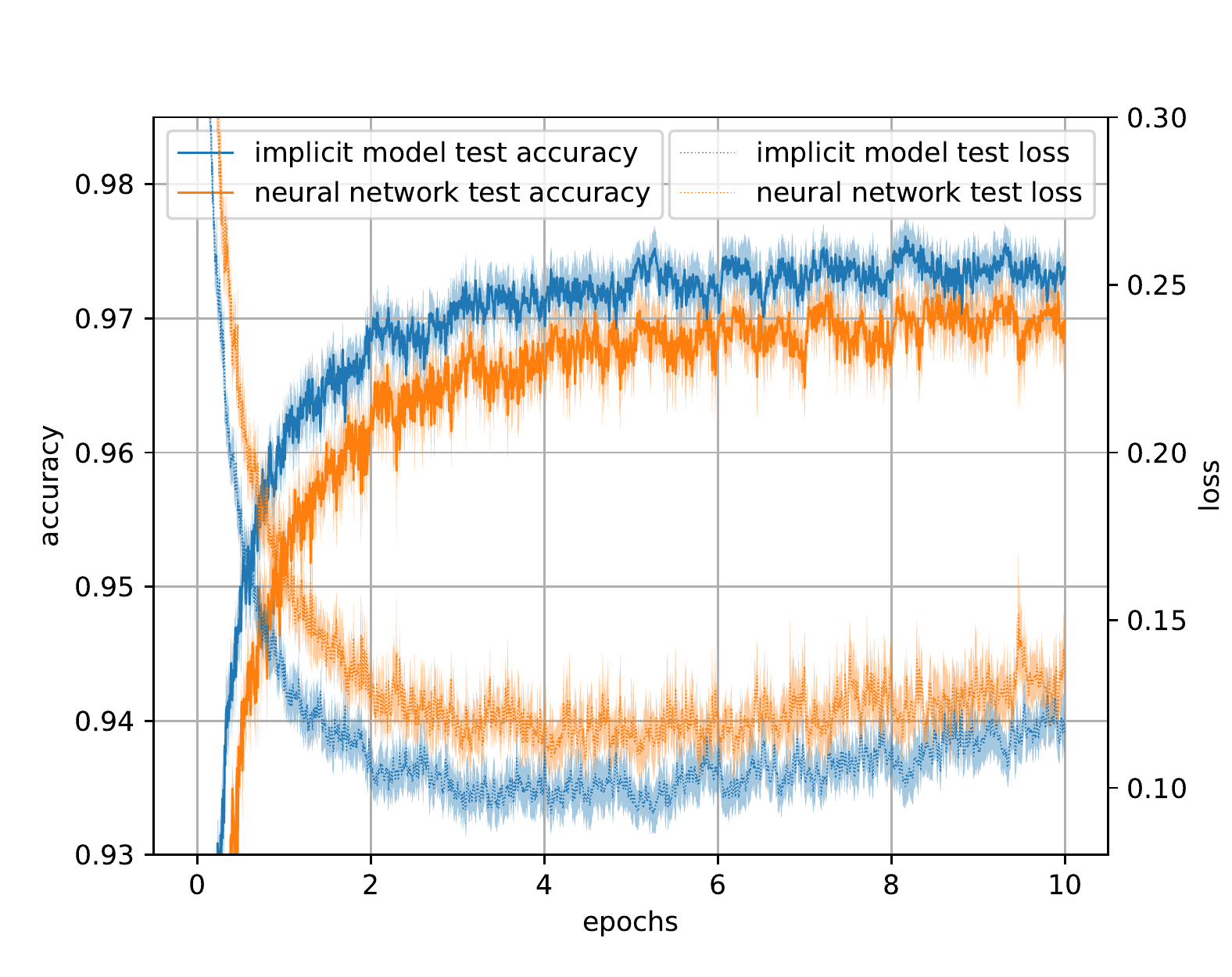}
\caption{Performance comparison on MNIST. Average best accuracy, implicit: 0.976, neural networks: 0.972. The curves are generated from 5 different runs with the lines marked as mean and region marked as the standard deviation over the runs.}
\label{fig:plotmnist}
\end{minipage}
\hfill
\begin{minipage}[t]{0.48\textwidth}
\centering
\includegraphics[width=1\columnwidth]{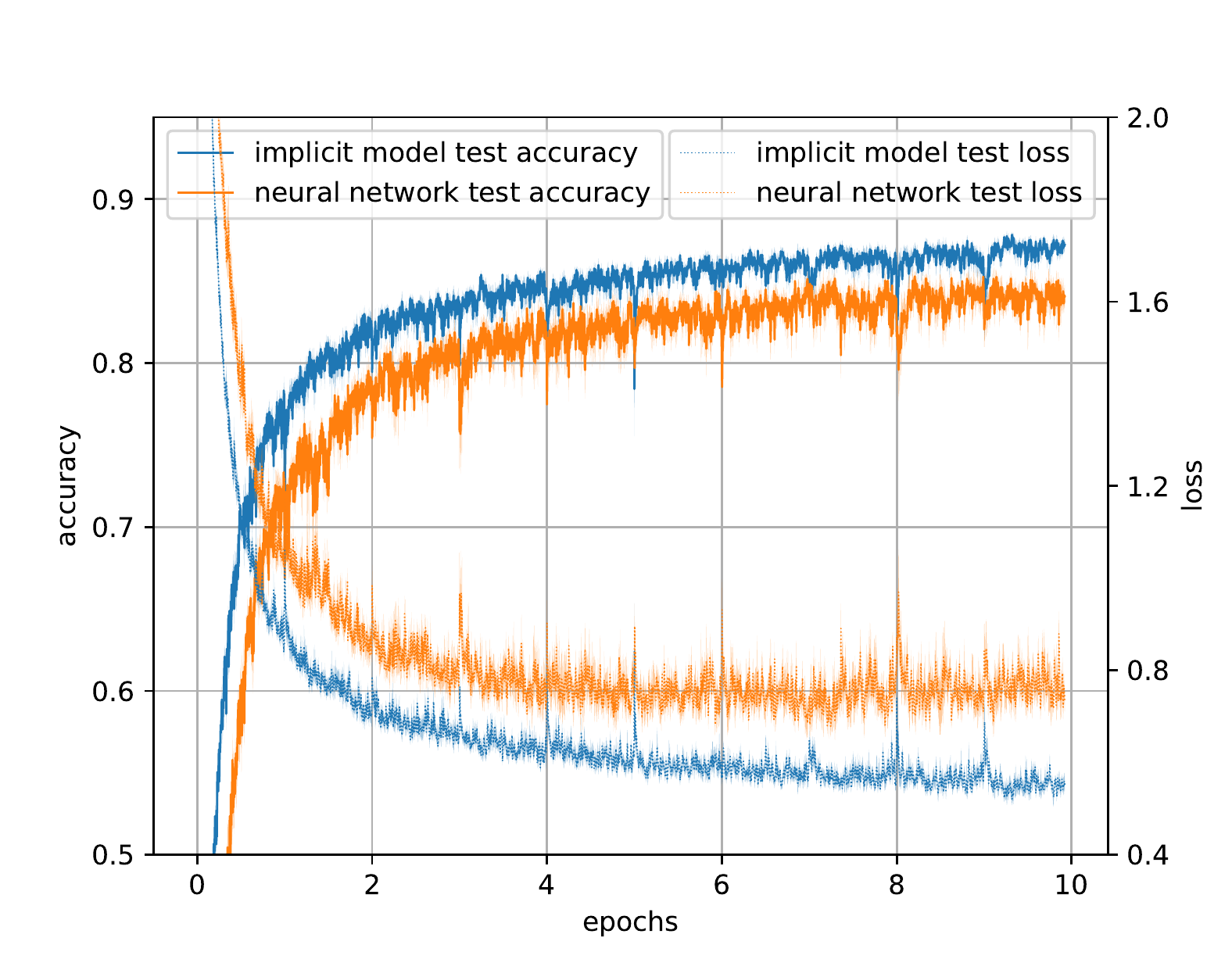}
\caption{Performance comparison on GTSRB.  Average best accuracy, implicit: 0.874, neural networks: 0.859. The curves are generated from 5 different runs with the lines marked as mean and region marked as the standard deviation over the runs.}
\label{fig:plottraffic}
\end{minipage}
\end{figure}

Similar to what we observed with synthetic datasets, the implicit model is capable of  matching and even outperform classical neural networks.

\subsection{Adversarial attack}

\subsubsection{Attack via the sensitivity matrix}
Our analysis above highlights the use of \emph{sensitivity matrix} as a measure for robustness. In this section, we show how sensitivity matrix can be used to generate effective attacks on two public datasets, MNIST and CIFAR-10. We compare our method against commonly used gradient-based attacks \cite{GoodfellowSS14, PapernotMJFCS16}. In this experiment, we consider two models: 1) feed-forward network trained on the MNIST dataset (98\% clean accuracy) and 2) ResNet-20 \cite{He2016DeepRL} trained on the CIFAR-10 dataset (92\% clean accuracy).

\begin{figure}[!h]
    \centering
    \includegraphics[width=.3\textwidth]{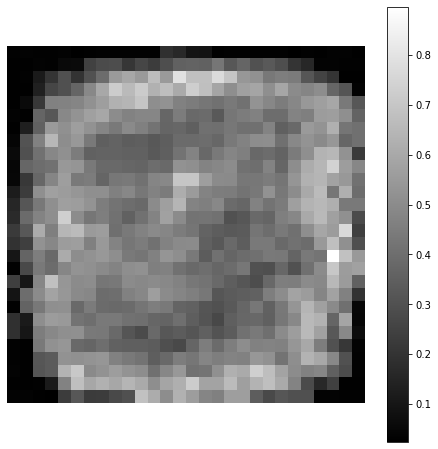}
    \includegraphics[width=.3\textwidth]{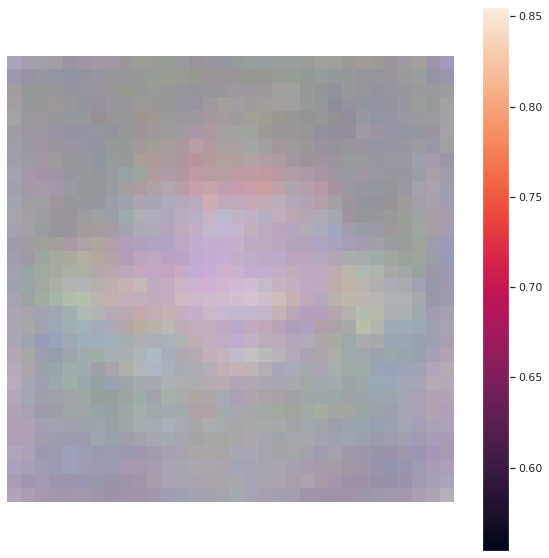}
    \caption{\emph{Left}: sensitivity values of a feed-forward network for the class ``digit 0'' in MNIST. \emph{Right}: sensitivity values of a ResNet-20 model for the class ``airplane'' in CIFAR-10. Brighter colors correspond to higher sensitivity when perturbed.}
    \label{fig:sensitivity-matrix}
\end{figure}

Figure \ref{fig:sensitivity-matrix} shows the sensitivity values for a particular class on MNIST and CIFAR-10 dataset. The sensitivity values are obtain from the implicit representation of the feed-forward network and ResNet-20. The 784 and 3072 input dimensions are arranged to correspond to the $28\times 28$ (grey scale) and $32\times 32$ (color) image pixel alignment. Brighter colors correspond to features with a higher impact on the output when perturbed. As a result, the sensitivity matrix can be used to generate adversarial attacks. 

We compare our method with commonly used gradient-based attacks. Precisely, for a given function $F$ (prediction rule) learned by a deep neural network, a benign sample $u \in \reals^p$ and the target $y$ associated with $u$, we compute the gradient of the function $F$ with respect to the given sample $u$, $\nabla_u F(u, y)$. We then take the absolute value of the gradient as an indication of which input features an adversary should perturb, similar to the saliency map technique \cite{SimonyanVZ13, PapernotMJFCS16}. The absolute value of the gradient can be seen as a ``local'' version of the sensitivity matrix; however,  unlike the gradient, the sensitivity matrix does not depend on the input data, making it a more general measurement of robustness for any given model.

\cref{tab:result} presents the experimental results of an adversarial attack using the sensitivity matrix and the absolute value of gradient on MNIST and CIFAR-10. For sensitivity matrix attack, we start from perturbing the input features that have the highest values according to the sensitivity matrix. For gradient-based attack, we do the same according to the absolute value of the gradient. We perturb the input features into small random values. Our experiments show that the sensitivity matrix attack is as effective as the gradient-based attack, while being very simple to implement. 

Interestingly, our attack does not rely on any input samples that the gradient-based attack needs. An adversary with the model parameters could easily craft adversarial samples using the sensitivity matrix. In the absence of access to the model parameters, an adversary can rely on the principle of \emph{tranferrability} \cite{LiuCLS17} and train a surrogate model to obtain the sensitivity matrix. \cref{fig:attack-examples} displays adversarial images generated using the sensitivity matrix. An interesting case is to use the sensitivity matrix to generate a sparse attack as seen in \cref{fig:attack-examples}.

\begin{table}[!h]
    \centering
    \caption{Experimental results of attack success rate against percentage of perturbed inputs on MNIST and CIFAR-10 (\numprint{10000} samples from test set).}
    \begin{tabular}{ lcccc }
     \toprule
     & \multicolumn{2}{c}{\textbf{Sensitivity matrix attack}} & \multicolumn{2}{c}{\textbf{Gradient-based attack}} \\
     \% of perturbed inputs & MNIST & CIFAR-10 & MNIST & CIFAR-10 \\
     \midrule
     0.1\% & 1.01\% & 3.04\% & 2.42\% & 1.75\% \\
     1\% & 13.41\% & 10.16\% & 26.92\% & 6.66\% \\
     10\% & 70.67\% & 36.21\% & 74.90\% & 33.18\% \\
     20\% & 89.82\% & 57.01\% & 87.10\% & 52.57\% \\
     30\% & 90.22\% & 67.45\% & 89.82\% & 66.59\% \\
     \bottomrule
    \end{tabular}
    \label{tab:result}
\end{table}

\begin{figure}[!h]
    \centering
    \includegraphics[width=.45\textwidth]{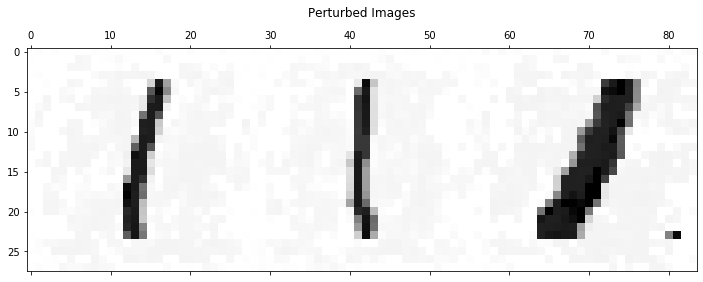}
    \includegraphics[width=.45\textwidth]{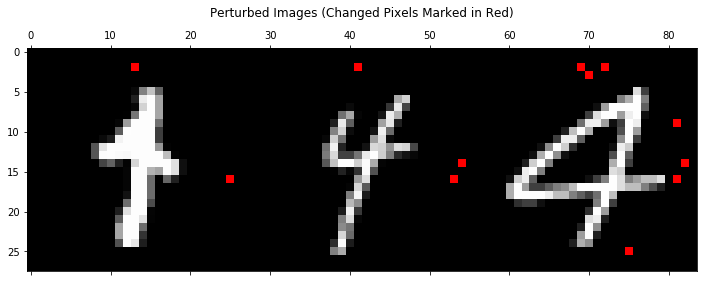}
    \includegraphics[width=.45\textwidth]{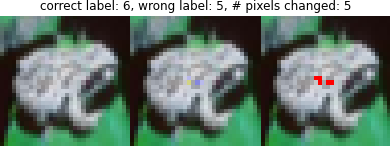}
    \includegraphics[width=.45\textwidth]{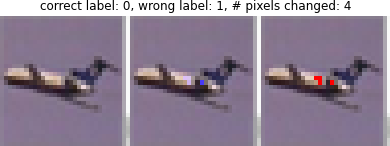}
    \caption{\emph{Top}: adversarial samples from MNIST. On the left are dense attacks with small perturbations and on the right are sparse attacks with random perturbations (perturbed pixels are marked as red). \emph{Bottom}: example sparse attack on CIFAR-10. The left ones are cleaned images, the middle ones are perturbed images, and the right ones mark the perturbed pixels in red for higher visibility.}
    \label{fig:attack-examples}
\end{figure}

\subsubsection{Attack with LP relaxation for CONE maps}
Although one can use sensitivity matrix to generate an effective adversarial example, one may wish to perform a more sophisticated attack by exploiting the weakness of an individual data point. This can be done by considering the LP relaxation (\cref{thm:lp-bound}), which has the advantage of generating a specific adversarial example for a given input data. The experiment in this section is again on MNIST and CIFAR-10 images. Here, the problem outlined in \cref{eq:state-wc-l1} is solved by LP relaxation, with the function $f_i(\xi) = (\xi - x_i^0)^2$. The optimization problem then finds a perturbed image that leads to the largest discrepancy between the perturbed state $x$ and the nominal state $x^0$. \cref{fig:cifar-attack} shows five example images. The clean images are shown in the top row. The perturbed images generated by the LP relaxation, as shown in the bottom row, appear quite similar to the naked eye; however, the model fails to predict these images correctly. The examples shown here are all correctly predicted by the model when no perturbation is added.

Our framework also allows for sparse adversarial attack by adding a cardinality constraint. \cref{fig:mnist-attack} shows three examples of perturbed images under non-sparse and sparse attack. Images on the left are the results of non-sparse attack and those on the right are the results of sparse attack. The model fails to predict the label correctly under both conditions. These results illustrate how the implicit prediction rules can be used to generate powerful adversarial attacks. It is also useful for adversarial training as a large amount of adversarial examples can be generated using the technique and be added back to the training data.

\begin{figure}[htbp]
    \centering
    \includegraphics[width=0.15\columnwidth]{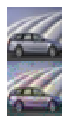}
    \includegraphics[width=0.15\columnwidth]{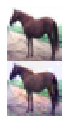}
    \includegraphics[width=0.15\columnwidth]{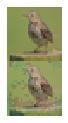}
    \includegraphics[width=0.15\columnwidth]{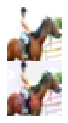}
    \includegraphics[width=0.15\columnwidth]{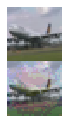}
    \caption{Example attack on CIFAR dataset. Top: clean data. Bottom: perturbed data.}
    \label{fig:cifar-attack}
\end{figure}
    
\begin{figure}[htbp]
    \centering
    \includegraphics[width=0.45\columnwidth]{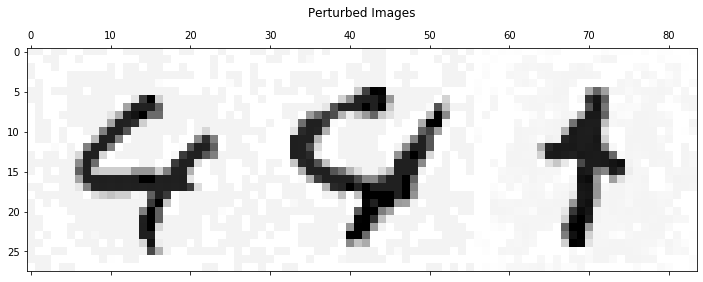}
    \includegraphics[width=0.45\columnwidth]{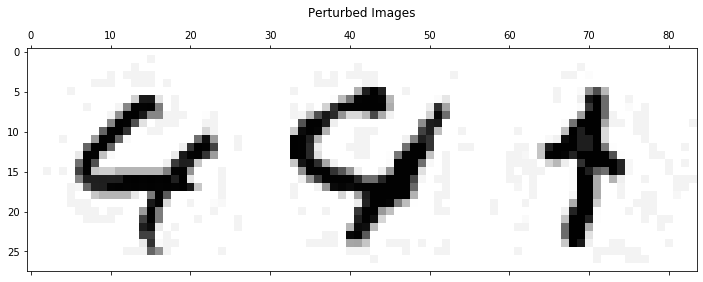}
    \caption{Example attack on MNIST dataset. Left: non-sparse attack. Right: sparse attack.}
    \label{fig:mnist-attack}
\end{figure}

\section{Prior Work}
\label{sec:related_work}

\subsection{In implicit deep learning} 
Recent works have considered versions of implicit models, and demonstrated their potential in deep learning. In particular, recent ground-breaking work by Kolter and collaborators \cite{zico19,kolter19} demonstrated success of an entirely implicit framework, which they call Deep Equilibrium Models, for the task of sequence modeling. Paper~\cite{node} uses implicit methods to solve and construct a general class of models known as neural ordinary differential equations, while \cite{NIPS2018_7948} uses implicit models to construct a differentiable physics engine that enables gradient-based learning and high sample efficiency. Furthermore, many papers explore the concept of integrating implicit models with modern deep learning methods in a variety of ways. For example, \cite{pmlr-v97-wang19e} show promise in integrating logical structures into deep learning by incorporating a semidefinite programming (SDP) layer into a network in order to solve a (relaxed) MAXSAT problem; see also \cite{wang2019satnet}. In  \cite{amos2018differentiable} the authors propose to include a model predictive control as a differentiable policy class for deep reinforcement learning, which can be seen as a kind of implicit architecture. In~\cite{amos2017optnet} the authors introduced implicit layers where the activation is the solution of some quadratic programming problem; in \cite{donti2017task}, the authors incorporate stochastic optimization formulation for end-to-end learning task, in which the model is trained by differentiating the solution of a stochastic programming problem. 

\subsection{In lifted models} 
\label{sub:in_lifted_models}
In implicit learning, there is usually no way to express the state variable in closed-form, which makes the task of computing gradients with respect to model parameters challenging.  Thus, a natural idea in implicit learning is to keep the state vector as a variable in the training problem, resulting in a higher-dimensional (or, ``lifted'') expression of the training problem. The idea of lifting the dimension of the training problem in (non-implicit) deep learning by introducing ``state'' variables has been studied in a variety of works; a non-extensive list includes \cite{taylor2016training}, \cite{askari2018lifted}, \cite{gu2018fenchel}, \cite{zeng2018global}, \cite{Zhang:2017}, \cite{carreira-perpinan14} and \cite{li2019lifted}. Lifted models are trained using block coordinate descent methods, Alternating Direction Method of Multipliers (ADMM) or iterative, non-gradient based methods. In this work, we introduce a novel aspect of lifted models, namely the possibility of defining a prediction rule implicitly. 

\subsection{In robustness analysis} 
\label{par:in_robustness_analysis}
The issue of robustness in deep learning is generating quite a bit of attention, due to the fact that many deep learning models suffer from the lack of robustness. Prior relevant work have demonstrated that deep learning models are vulnerable to adversarial attacks \cite{GoodfellowSS14, KurakinGB17a, PapernotMJFCS16, Kurakin:2017}. The vulnerability issue of deep learning models have motivated the study of corresponding defense strategies \cite{madry2018towards, PapernotM0JS16, NIPS2018_8285, Gowal2018, Shaham:2015, Wong:2017, Cohen:2019}. However, many of the defense strategies are later shown to be ineffective \cite{AthalyeC018, Carlini017-adversarial}, suggesting the needs for the theoretical understanding of robustness evaluations for deep learning model. In this work, we formalize the robustness analysis of deep learning via the lens of the implicit model. A large number of deep learning architectures can be modeled using implicit prediction rules, making our robustness evaluation a versatile analysis tool. 

In \cite{NIPS2018_8285}, the authors arrive at a similar formulation (which we discuss below) but do not explore the quality of the attacks that can be extracted by solving \eqref{eq:bidual}.

The SDP-based formulation given in~\cref{sub:sdp_relaxations} parallels the formulation in \cite{NIPS2018_8285} where the authors represent interval sets ($l \leq x \leq u$) as the quadratic constraint $x \circ x \leq (l+u) \circ x - l \circ u$ and work on the rank-relaxed version of their problem to arrive at a SDP. Upon further inspection, our formulation is similar to, and therefore extends, the formulation proposed in \cite{NIPS2018_8285} for feed-forward neural networks. Indeed, for a neural network, $A$ is block strictly upper triangular and $(I-|A|)^{-1}$ has a closed form expression; in the case of three layers we have:
\begin{align*}
    &A = \begin{bmatrix}
    0 & W_2 & 0\\
    0 & 0 & W_1 \\
    0 & 0 & 0
    \end{bmatrix}, \;\; B = \begin{bmatrix}
    0 \\ 
    0 \\ 
    W_0
    \end{bmatrix} ,
    \end{align*}
    leading to
    \begin{align*}
    (I-|A|)^{-1}|B| &= \begin{bmatrix}
    I & |W_2| & |W_2||W_1| \\
    0 & I & |W_1| \\
    0 & 0 & I
    \end{bmatrix} \begin{bmatrix}
     0 \\ 
    0 \\ 
    |W_0|
    \end{bmatrix} = \begin{bmatrix}
    |W_2||W_1||W_0| \\
    |W_1||W_0| \\
    |W_0|
    \end{bmatrix}.
\end{align*}
The difference between our proposed SDP and the SDP in \cite{NIPS2018_8285} is how the input-uncertainty constraints are dealt with at every layer. In \cite{NIPS2018_8285}, the authors strengthen their formulation by reducing the size of the uncertainty set at each layer. That is, they propagate the input uncertainty and bound the size of the uncertainty at each layer of the neural network (see Sections 5.1 and 5.2 of \cite{NIPS2018_8285}). As written, \eqref{eq:bidual} does not include these constraints and hence can be tightened by adding extra box-constraints as is done in the aforementioned reference. This allows to extend, to implicit models, the results that were derived for traditional feed-forward networks. 


\subsection{In sparsity, compression and deep feature selection}
Sparsity and compression, which are well understood in classical settings, have found their place in deep learning and are an active branch of research. Early work in pruning dates back to as early as the 90s \cite{lecun1990optimal, hassibi1993optimal} and has since gained interest. In \cite{srivastava2014dropout}, the authors showed that by randomly dropping units (\ie\ increasing the sparsity level of the network or compressing the network) reduces overfitting and improved the generalization performance of networks. Recently, more sophisticated ways of pruning networks have been proposed, in an effort to reduce the overall size of the model, while retaining or accepting a modest decrease in accuracy: a non-extensive list of works include \cite{zhu2017prune, narang2017exploring, han2015deep, see2016compression, anwar2017structured, lebedev2016fast, changpinyo2017power, louizos2017learning, evci2019difficulty}. Sparsity and model compression is very closely tied with other techniques to reduce model size. This includes, but is not limited to, low-rank matrix factorization, group sparsity regularization, quantization techniques and low precision networks. While this sub-branch of research is far from being complete, sparsity and compression offer promise in the way of low memory footprints and generalization performance for deep networks.

\section{Concluding Remarks}
\label{sec:conclusions}
We conclude with a few perspectives for future research on implicit models.

\subsection{Cousins of implicit models}
\begin{figure}[h]
\centering
\includegraphics[width=.45\textwidth]
{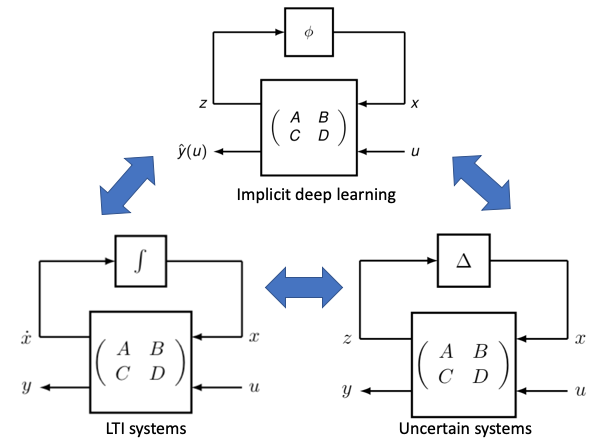}
\caption{Some cousins of implicit models: LTI systems (bottom left) and uncertain systems (bottom right).}
\end{figure}

Implicit models rely on a representation of the prediction rule where the linear operations are clearly separated from the (parameter-free) nonlinear ones, leading to the block-diagram in~\cref{fig:Figures_BlockDiag}.

This idea is strongly reminiscent of earlier representations arising in systems and control theory. Perhaps the most famous example lies with linear time-invariant (LTI) systems, where the idea is essentially equivalent to the state-space representation of transfer functions~\cite{aplevich2000essentials}. In that context, the block-diagram representation involves an integrator (in continuous or discrete time) in place of the static nonlinear map $\phi$.  The idea is also connected to the linear-fractional representation (or, transformation, LFT) arising in uncertain systems~\cite{dullerud2013course}, where now the linear operations actually represent a dynamic system (that is, the matrix $M$ is an LTI system itself), and the map $\phi$ is replaced by an uncertain matrix of parameters or LTI system. 

These connections raise the prospect of a unified theory tackling deep networks inside the loop of a dynamical system, where such networks can be composed (via feedback connections) with LTI or more generally uncertain systems. With the machinery of Linear Matrix Inequalities or the more general Integral Quadratic Constraints framework~\cite{megretski1997system}, one can develop rigorous analyses performed on systems with deep networks in the loop, so that bounds on, say, stability margins can be computed.

\subsection{Minimal representations}
The prediction rule $u \rightarrow \hat{y}(u)$ provided by a given implicit model can be in general obtained by many alternative models---scaling the model matrices via the diagonal scaling given in~\cref{eq:collatz-wiedlandt} illustrates this non-unicity. An interesting theoretical study would focus on obtaining representations that are minimal from the point of view of size, precisely the order $n$, which would extend the well-established theory of minimal representations of LTI systems.

\section*{Acknowledgments}
We would like to acknowledge useful comments and suggestions by Mert Pilanci and Zico Kolter.

\bibliographystyle{siamplain}
\bibliography{idl}

\newpage
\appendix

\section[Proof of Thm]{Proof of~\cref{thm:pf-wp}} 
\label{app:proof_of_wellposedness}
Let $b \in \reals^n$. We first prove the existence of a solution $\xi \in \reals^n$ to the equation $\xi = \phi(A\xi + b)$. Consider the Picard iteration~\cref{eq:fixed-point-b}. We have for every $t \ge 1$:
\[
|x(t+1)-x(t)| = |\phi(Ax(t)+b) - \phi(Ax(t-1)+b)| \le |A| |x(t) - x(t-1)|,
\]
which implies that for every $t,\tau \ge 0$:
\[
|x(t+\tau) - x(t)| \le \sum_{k=t}^{t+\tau} |A|^k |x(1) - x(0)| \le |A|^t \sum_{k=0}^{\tau} |A|^k |x(1) - x(0)| \le |A|^t w,
\]
where 
\[
w := \sum_{k=0}^{+\infty} |A|^k |x(1) - x(0)| = (I-|A|)^{-1}|x(1) - x(0)| .
\]
Here we exploited the fact that, due to $\lpf(|A|) <1$, $I-|A|$ is invertible, and the series above converges. Since $\lim_{t \rightarrow 0} |A|^t = 0$, we obtain that $x(t)$ is a Cauchy sequence, hence it has a limit point, $x_\infty$. By continuity of $\phi$ we further obtain that $x_\infty = \phi(Ax_\infty + b)$, which establishes the existence of a solution.

To prove unicity, consider $x^1,x^2 \in \reals^n_+$ two solutions to the equation. Using the hypotheses in the theorem, we have, for any $k \ge 1$:
\[
|x^1-x^2| \le |A| |x^1-x^2| \le |A|^k |x^1-x^2|.
\]
The fact that $|A|^k \rightarrow 0$ as $k \rightarrow +\infty$ then establishes unicity.

\section[Proof of Thm]{Proof of~\cref{thm:pf-wp-block}} 
\label{app:proof_of_wellposedness-block}
We follow the steps taken in the proof of~\cref{thm:pf-wp}. 
Let $b \in \reals^n$. Our first step is to establish that for the Picard iteration~\cref{eq:fixed-point-b}, we have
\[
\eta(x(t+1)-x(t)) \le M \eta(x(t+1)-x(t))
\]
for an appropriate non-negative matrix $M \ge0$.  Here, $\eta$ is a vector of norms, as defined in~\cref{eq:vect-norms}. For every $l \in [L]$, $t \ge 1$:
\begin{align*}
[\eta(x(t+1)-x(t))]_l &= \|[\phi(A(x(t)+b)- \phi(Ax(t-1)+b)]_l\|_{p_l} \\
&= \|\phi_l((A(x(t)+b)_l)- \phi_l((Ax(t-1)+b)_l)\|_{p_l} \\
&\le \gamma_l \|[A(x(t)-x(t-1))]_l\|_{p_l} \\
&= \gamma_l \|\sum_{h \in [L]}  A_{lh} (x(t)-x(t-1))_h\|_{p_l} \\
& \le \gamma_l \sum_{h \in [L]} \|A\|_{p_l \rightarrow p_h} \|x_h(t+1)-x_h(t)\|_{p_h} = \gamma_l [M \eta(x(t+1)-x(t))]_l, 
\end{align*}
which establishes the desired bound, with $M:= \diag(\gamma)N(A)$, where $\gamma$ is the vector of Lipschitz constants, and $N(\cdot)$ is the $L \times L$ matrix of induced norms~\cref{eq:induced-norm-matrix}. 

Assume now that $\lpf(M) <1$, as posited in the Theorem. Then $I-M$ is invertible. We proceed to prove existence by showing that the sequence of Picard iterates is Cauchy: for every $t,\tau \ge 0$,
\[
\eta(x(t+\tau) - x(t)) \le \sum_{k=t}^{t+\tau} M^k \eta(x(1) - x(0)) \le M^t \sum_{k=0}^{\tau} M^k \eta(x(1) - x(0)) \le M^t w,
\]
where 
\[
w := \sum_{k=0}^{+\infty} M^k \eta(x(1) - x(0)) = (I-M)^{-1}\eta(x(1) - x(0)) .
\]
The above proves the existence.

To prove unicity, consider $x^1,x^2 \in \reals^n_+$ two solutions to the equation. Using the hypotheses in the theorem, we have, for any $k \ge 1$:
\[
\eta(x^1-x^2) \le M \eta(x^1-x^2) \le M^k \eta(x^1-x^2).
\]
The fact that $M^k \rightarrow 0$ as $k \rightarrow +\infty$ then establishes unicity.

\section{Proof of~\cref{thm:wp-tri}} 
\label{app:wp-tri}
Express the equation $x = \phi(Ax+b)$ as
\[
x_1 = \phi(A_{11}x_1+A_{12}x_2+b_1), \;\; x_2 = \phi(A_{22} x_2 + b_2),
\]
where $b = (b_1,b_2)$, $x=(x_1,x_2)$, with $b_i \in \reals^{n_i}$, $x_i \in \reals^{n_i}$, $i=1,2$. Here, since $\phi$ acts componentwise, we use the same notation $\phi$ in the two equations.

Now assume that $A_{11}$ and $A_{22}$ are well-posed with respect to $\phi$.  Since $A_{22}$ is well-posed for $\phi$, the second equation has a unique solution $x_2^*$; plugging $x_2=x_2^*$ into the second equation, and using the well-posedness of $A_{11}$, we see that the first equation has a unique solution in $x_1$, hence $A$ is well-posed.

To prove the converse direction, assume that $A$ is well-posed. The second equation above must have a unique solution $x_2^*$, irrespective to the choice of $b_2$, hence $A_{22}$ must be well-posed.  To prove that $A_{11}$ must be well-posed too, set $b_2 = 0$, $b_1$ arbitrary, leading to the system
\[
x_1 = \phi(A_{11} x_1 + A_{12}x_2 + b_1), \;\; 
x_2 = \phi(A_{22}x_2) .
\]
Since $A_{22}$ is well-posed for $\phi$, there is a unique solution $x_2^*$ to the second equation; the first equation then reads $x_1 = \phi(A_{11} x_1 + b_1+A_{12}x_2^*)$. It must have a unique solution for any $b_1$, hence $A_{11}$ is well-posed.

\section{Proofs of~\cref{thm:box-bnd}, \cref{thm:box-bnd-blip} and~\cref{thm:box-bnd-output-blip}} 
\label{app:box-state-bnds}
First let us prove~\cref{thm:box-bnd}.
Due to the non-expansivity of the CONE map $\phi$:
\begin{align}\label{eq:state-bounds-cone-ineq}
|x-x^0| \le |A||x-x^0| + b,
\end{align}
where $b := |B| \sigma_u \ge 0$. Since $\lpf(|A|) <1$, the matrix $I-|A|$ is invertible and its inverse is componentwise non-negative; for any $n$-vector $b \ge 0$, we have 
\[
\bar{\sigma} := (I-|A|)^{-1}b = \sum_{k\ge0} |A|^k b \ge 0.
\] 
Further, any $n$-vector $\sigma$ such that $\sigma \le |A|\sigma + b$,
we have, for every integer $N$:
\[
\sigma \le |A|^N \sigma + \sum_{k=0}^{N-1} |A|^k b \le |A|^N \sigma + \bar{\sigma}.
\]
Since $\lpf(|A|) <1$, we have $|A|^N \rightarrow 0$ as $N \rightarrow +\infty$, which yields $\sigma \le \bar{\sigma}$.
Applying this to $\sigma = |x-x^0|$ yields the desired result.

Now turn to the proof of~\cref{thm:box-bnd-blip}. 
For every $l \in [L]$:
\begin{align*}
[\eta(x-x^0)]_l & \le \|[\phi(A(x-x^0) + B(u-u^0)b)]_l\|_{p_l} \\
&\le \gamma_l \|\sum_{h \in [L]}  A_{lh} (x-x^0)_h\|_{p_l} + \gamma_l \|\sum_{i \in [p]}  B_{li} (u-u^0)_i\|_{p_l}\\
& \le \gamma_l [N(A) \eta(x-x^0))]_l + \gamma_l \sum_{i \in [p]}  \|B_{li}\|_{p_l} |u-u^0|_i\\
& \le \gamma_l [N(A) \eta(x-x^0))]_l + \gamma_l [N(B) |u-u^0|]_l, 
\end{align*}
which establishes the desired bound.

The result of~\cref{thm:box-bnd-output-blip} is obtained similarly. For given $i \in [q]$, we have
\begin{align*}
    |\hat{y}(u)-\hat{y}(u^0)|_i &\le | \sum_{l \in [L]} C_{il} (x-x^0)_l| + (|D| |u-u^0|)_i \\
    &\le \sum_{l \in [L]} \|C_{il}\|_{p^*_l} \eta(x-x^0)_l + (|D| |u-u^0|)_i.
\end{align*}

\section[Proof of Thm]{Proof of~\cref{thm:lp-bound}} 
\label{app:proof_of_lp_bnds}
We have
\begin{eqnarray*}
p^* \le \overline{p} &:=& \min_{\lambda} \: \max_{x, \: u \in {\cal U}} \: \sum_{i \in [n]} f_i(x_i) + \lambda^\top (Ax+Bu-z) ~:~ x = z_+, \;\; |x-x^0| \le \sigma_x \\
&=& \min_{\lambda} \: \max_{u \in {\cal U}} \: \lambda^\top Bu + \max_{z \::\:  |z^+-x^0| \le \sigma_x} \: \sum_{i \in [n]} ( f_i(z_i^+) + (A^\top \lambda)_i z_i^+ - \lambda_i z_i ) \\
&=& \min_{\lambda, \: \mu=A^\top \lambda} \: \left(\max_{u \in {\cal U}} \: \lambda^\top  Bu \right) + \sum_{i \in [n]} g_i(\lambda_i,\mu_i) ,
\end{eqnarray*}
which establishes the first part of the theorem. If we further assume that the functions $g_i$, $i \in [n]$ are closed, strong duality holds, so that
\begin{eqnarray*}
\overline{p}	
&=& \min_{\lambda, \: \mu} \: \max_{x, \: u \in {\cal U}} \: \lambda^\top  Bu
+ x^\top (A^\top \lambda-\mu) + \sum_{i \in [n]} g_i(\lambda_i,\mu_i) \\
&=& 
\max_{x, \: u \in \Co {\cal U}} \: -\sum_{i \in [n]} g_i^*(-(Ax+Bu)_i,x_i)	
\end{eqnarray*}
where $g_i^*$ is the conjugate of $g_i$, $i \in [n]$.

\section{Proof of~\cref{thm:pf-pertA}} 
\label{app:proof_pertA}
We have
\begin{align*}
|x-x^0|
&\le|Ax - Ax^0| = |A^0(x-x^0) + E(x-x^0) + Ex^0|
&\le |A^0+E||x-x^0| + |E| |x^0| .
\end{align*}
Applying a technique similar to that employed in the proof of~\cref{thm:box-bnd}, we obtain the desired relative error bounds.

\section{Gradient Equations} 
\label{app:grad-descent}
In this section, we detail gradient computations with respect to the model parameter matrix $M=(A,B,C,D)$, in the context of the training problem~\ref{eq:training-pb-linf}. We assume that the map $\phi$ is BLIP with parameters $n_l$, $p_l$, $\gamma_l$, $l \in [L]$, and that $A$ satisfies the corresponding PF well-posedness condition with respect to $\phi$; we also assume that the latter is differentiable.

For simplicity, we first consider mini-batches of size $1$ (that is, $X = x \in \mathbb{R}^n$ and $U = u \in \mathbb{R}^p$), and we define $\hat{y} = Cx+Du $, $z = Ax + Bu$. We wish to calculate 
\[
\nabla_M \mathcal{L} = \begin{pmatrix}
\nabla_A \mathcal{L} &
\nabla_B \mathcal{L} \\
\nabla_C \mathcal{L} &
\nabla_D \mathcal{L} \\
\end{pmatrix}.
\]
The difficult part are the terms $\nabla_A \mathcal{L}$ and $\nabla_B \mathcal{L}$ due to the presence of the equilibrium equality constraint. We deal with this via implicit differentiation. In order to establish the existence of gradients we will need the following lemma.
\begin{lemma} \label{lem:d_wellposed}
Assume that the map $\phi$ is a BLIP map with parameters $n_l$, $p_l$, $\gamma_l$, $l \in [L]$, and that $A$ is well-posed with respect to $\phi$. Define the block-diagonal matrix $\Phi := \frac{\partial \phi(z)}{\partial z}$, then the equation in the $n \times n$ matrix $G$: $G = \Phi (AG+I)$ has a unique solution, which can be computed as the limit point of the recursion
\begin{equation}
    \label{eq:grad-fixed-point}
G(t+1) = \Phi (AG(t)+I), \;\; t=0,1,2,\ldots
\end{equation}

\end{lemma}
\begin{proof}
The result follows from the fact that if $\phi$ is BLIP with parameters $n_l$, $p_l$, $\gamma_l$, $l \in [L]$, then $\Phi = \diag(\Phi_1,\ldots,\Phi_L)$, where $\|\Phi_l\|_{n_l} < 1$, $l \in [L]$ and \cref{thm:pf-wp-block}.
\end{proof}

\paragraph{Calculating $\nabla_A \mathcal{L}$} For a given index pair $(j,k)$, we have 
\begin{align*}
\dfrac{\partial \mathcal{L}}{\partial A_{jk}} &= \dfrac{\partial \mathcal{L}}{\partial z} \cdot \dfrac{\partial Ax + Bu}{\partial A_{jk}} = \dfrac{\partial \mathcal{L}}{\partial z} \cdot \dfrac{\partial \sum_h A_{lh} x_h}{\partial A_{jk}}  = \dfrac{\partial \mathcal{L}}{\partial z} e_j x_k = \left[\nabla_z \mathcal{L} \; x^\top\right]_{jk} .
\end{align*}
We calculate $\nabla_z \mathcal{L}$ via implicit differentiation:
\begin{align} 
	\nabla_z \mathcal{L} & = \left(\frac{\partial \mathcal{L}}{\partial x} \cdot \frac{\partial x}{\partial z} \right)^\top , \;\;
	\frac{\partial\mathcal{L}}{\partial x} = \frac{\partial\mathcal{L}}{\partial\hat{y}} \cdot \frac{\partial Cx+Du}{\partial x}, \notag\\
	\frac{\partial x}{\partial z} & = \frac{\partial \phi(z)}{\partial z} + \frac{\partial \phi(Ax + Bu)}{\partial x} \cdot \frac{\partial x}{\partial z} \label{eq:grad_key} = (I - \Phi A)^{-1} \Phi, \notag
\end{align}
where $\Phi := \frac{\partial \phi(z)}{\partial z}$ is a block diagonal matrix. Thanks to \cref{lem:d_wellposed}, the matrix $G:=(I-\Phi A)^{-1}\Phi$ exists and can be obtained via fixed-point iterations~\ref{eq:grad-fixed-point}. Note that the gradient of the loss function $\nabla_{\hat{y}} \mathcal{L}$ can be easily computed, and we have
\begin{equation}
    \label{eq:nablazL}
\nabla_z \mathcal{L} = \left(C (I - \Phi A)^{-1} \Phi\right)^\top \nabla_{\hat{y}} \mathcal{L} = \left(CG\right)^\top \nabla_{\hat{y}} \mathcal{L} .
\end{equation}

\paragraph{Calculating the other gradients}
The other gradients are easily computed, as seen next, where $(j,k)$ denotes a generic index pair.
From the above it follows that
\begin{align*}
\dfrac{\partial \mathcal{L}}{\partial B_{jk}} &= \dfrac{\partial \mathcal{L}}{\partial z} \cdot \dfrac{\partial Ax + Bu}{\partial B_{jk}} = \left[\nabla_z \mathcal{L} \; u^\top\right]_{jk}  \Longrightarrow
\nabla_B \mathcal{L} =  \nabla_z {\mathcal{L}} \; u^\top,
\end{align*}
where $\nabla_z \mathcal{L}$  is given in \ref{eq:nablazL}.
Likewise,
\begin{align*}
\dfrac{\partial \mathcal{L}}{\partial C_{jk}} &= \dfrac{\partial \mathcal{L}}{\partial \hat{y}} \cdot \dfrac{\partial Cx + Du}{\partial C_{jk}} = \left[\nabla_{\hat{y}} \mathcal{L} \; x^\top\right]_{jk} \Longrightarrow
\nabla_C \mathcal{L} =  \nabla_{\hat{y}} \mathcal{L} \; x^\top ,
\end{align*}
and
\begin{align*}
\dfrac{\partial \mathcal{L}}{\partial D_{jk}} &= \dfrac{\partial \mathcal{L}}{\partial \hat{y}} \cdot \dfrac{\partial Cx + Du}{\partial D_{jk}} = \left[\nabla_{\hat{y}} \mathcal{L} \; u^\top\right]_{jk} \Longrightarrow
\nabla_D \mathcal{L} =  \nabla_{\hat{y}} \mathcal{L} \; u^\top .
\end{align*}
Finally, the gradient with respect to input $u$ is obtained as
\begin{align*}
    \frac{\partial \mathcal{L}}{\partial u_j} &=  \dfrac{\partial \mathcal{L}}{\partial z} \cdot \dfrac{\partial Ax + Bu}{\partial u_j} + \dfrac{\partial \mathcal{L}}{\partial \hat{y}} \cdot \dfrac{\partial Cx + Du}{\partial u_j} = \dfrac{\partial \mathcal{L}}{\partial z} \cdot B e_j + \dfrac{\partial \mathcal{L}}{\partial \hat{y}} \cdot D e_j \\
    &=  \left[B^\top \nabla_z \mathcal{L} + D^\top \nabla_{\hat{y}} \mathcal{L} \right]_{j}
    \Longrightarrow
\nabla_u \mathcal{L} =  B^\top \nabla_z \mathcal{L} + D^\top \nabla_{\hat{y}} \mathcal{L},
\end{align*}
where $\nabla_z \mathcal{L}$  is given in \ref{eq:nablazL}.

To summarize, we have obtained the following closed form evaluation of gradients for ($A, B, C, D$):
$$
\nabla_M \cal L = \begin{pmatrix}
\nabla_A \cal L &
\nabla_B \cal L \\
\nabla_C \cal L &
\nabla_D \cal L \\
\end{pmatrix} = \begin{pmatrix}
\nabla_z \mathcal{L} \\
\nabla_{\hat{y}} \mathcal{L}
\end{pmatrix}
\begin{pmatrix}
x \\ u 
\end{pmatrix}^\top,
$$
where $\nabla_{\hat{y}} \mathcal{L}$ can be easily computed and $\nabla_z\mathcal{L}$ is given by the following equation:
$$
\nabla_z\mathcal{L} = \left(C G\right)^\top \nabla_{\hat{y}} \mathcal{L}, \;\;
G = (I - \Phi A)^{-1} \Phi = \Phi (AG+I).
$$
Let $v:=\nabla_z \mathcal{L}$, we have
$$ v = \left(C (I - \Phi A)^{-1} \Phi\right)^\top \nabla_{\hat{y}} \mathcal{L},$$
which is equivalent to
\begin{equation} \label{eq:grad_fixpoint_z}
    v = \Phi \left(A^\top v + C^\top \nabla_{\hat{y}} \mathcal{L}\right).
\end{equation}
Again \cref{eq:grad_fixpoint_z} can be computed via the recursion~(\ref{eq:fixed-point-v}). In practice we use \cref{eq:grad_fixpoint_z} to compute $\nabla_z \mathcal{L}$ without explicitly forming  $G$ via \cref{eq:grad-fixed-point} where each iteration would involve a matrix-matrix product instead of a matrix-vector product in \cref{eq:grad_fixpoint_z}.

\paragraph{Mini-batch calculations} We may efficiently compute the gradient corresponding to a whole mini-batch of size greater than $1$. In the case of componentwise maps $\phi$, this is based on expressing the equation~(\ref{eq:grad_fixpoint_z}) as
\begin{equation} \label{eq:grad_fixpoint_z-mini}
    V = \Psi \circ \left(A^\top V + C^\top L\right),
\end{equation}
where columns of $L$ contains gradients of the loss with respect to $\hat{y}$, and $\Psi$ is a matrix that contains the derivatives of the activation, one column corresponding to the (diagonal) elements of $\Phi$, where $\Phi$ is defined in the previous section.
\section{Projection on $l_\infty$ Matrix Norm Ball} 
\label{app:vec-bisec-proj}
We address problem~\cref{eq:proj_nrm_ball}, which we write as
\[
p^*: = \min_{A} \: \onehalf\|A-A^0\|_F^2 ~:~ \sum_{j \in [n]} |A_{ij}|  \le \kappa, \;\; i \in [n].
\]
where $A^0 \in \reals^{n \times n}$ is  given. The problem is decomposable across the rows of the matrices involved, leading to $n$ sub-problems of the form
\[
\min_a \: \onehalf\|a-a^0_i\|_2^2 ~:~ \|a\|_1 \le \kappa,
\]
which $a^0_i \in \reals^n$ the $i$-th row of $A^0$.

The problem cannot be solved in closed form, but a bisection method can be applied to the dual:
\[
p^* = \max_{\lambda \ge 0} \: -\kappa\lambda + \sum_{i\in[n]} s_i(\lambda) ,
\]
where, for $\lambda \ge 0$ given:
\[
s_i(\lambda) := \min_\xi\: \onehalf(\xi-a^0_i)^2 + \lambda |\xi| , \;\; i \in [n].
\] 
A subgradient of the objective is 
\[
g_i(\lambda) := -\kappa + \sum_{i \in [n]} \max(|a^0_i|-\lambda,0), \;\; i \in [n].
\]
Observe that $p^* \ge 0$, hence at optimum:
\[
0 \le \lambda \le \frac{1}{\kappa}\sum_{i\in[n]} s(\lambda,a^0_i) \le \lambda^{\rm max} := \frac{1}{2\kappa}\|a^0_i\|_2^2.
\]
The bisection can be initialized with the interval $\lambda \in [0,\lambda^{\rm max}]$.

Returning to the original problem~\cref{eq:proj_nrm_ball}, we see that all the iterations can be expressed in a ``vectorized'' form, where updates for the different rows of $A$ are done in parallel. The dual variables corresponding to each row are collected in a vector  $\lambda \in \reals^n$. We initialize the bisection with a vector interval 
$[\lambda_l,\lambda_u]$, with 
$\lambda^l = 0$, $\lambda^u_i = \onehalf \|a^0_i\|_2^2/\kappa$, $i \in [n]$. We update the current vector interval as follows:
\begin{enumerate}
\item Set $\lambda = (\lambda_l+\lambda_u)/2$.
    \item Form a vector $g(\lambda)$ containing the sub-gradients corresponding to each row, evaluated at $\lambda_i$, $i \in [n]$:
    \[
    g(\lambda) = -\kappa\ones + (|A^0|-\lambda \ones^T)_+^T\ones.
    \]
    \item For every  $i \in [n]$, reset $\lambda^u_i = \lambda_i$ if $g_i(\lambda) >0$, $\lambda^l_i = \lambda_i$ if $g_i(\lambda) \le 0$.
\end{enumerate}

\end{document}